\documentclass[12pt,]{article}
\usepackage{scicite}
\usepackage{times}
\usepackage{textalpha} 
\usepackage{makecell}
\usepackage{pdflscape}
\usepackage{multirow}

\usepackage{xcolor}
\usepackage{amsmath} 
\usepackage{amssymb}
\usepackage{amsthm}
\usepackage{subcaption}
\usepackage{graphicx}
\usepackage{stix}
\usepackage[letterpaper,margin=1in]{geometry}
\usepackage{hyperref}
\usepackage{gensymb}

\usepackage{float}
\usepackage{xspace}
\usepackage[vlined,ruled,linesnumbered]{algorithm2e}
\usepackage{algorithmic}
\usepackage{enumerate}

\PassOptionsToPackage{end}{algorithmic}
\usepackage{cleveref}
\usepackage{pifont}    
 
\newcommand{\cmark}{\checkmark}
\newcommand{\xmark}{\ding{55}}
\usepackage{booktabs}   
\usepackage{longtable}  
\usepackage{colortbl}   
\usepackage{multirow}   
\usepackage{arydshln}   

\definecolor{goodgreen}{HTML}{1B7A2B}
\definecolor{badred}{HTML}{B71C1C}
\definecolor{nagray}{HTML}{888888}
 
\definecolor{hwbg}{HTML}{E8F0E5}
\definecolor{ctrlbg}{HTML}{E1E9F2}
\definecolor{implbg}{HTML}{F2E9DF}

\newcommand{\gc}{\textcolor{goodgreen}{\cmark}}
\newcommand{\bc}{\textcolor{badred}{\xmark}}

\definecolor{sc0}{HTML}{D32F2F}   
\definecolor{sc1}{HTML}{E8873A}   
\definecolor{sc2}{HTML}{A3B938}   
\definecolor{sc3}{HTML}{2E7D32}   
 
\newfloat{movie}{tbhp}{lom}
\floatname{movie}{Movie}
\crefname{movie}{movie}{movies}
\Crefname{movie}{Movie}{Movies}

\newcommand{\ie}{\emph{i.e.},\xspace}
\newcommand{\eg}{\emph{e.g.},\xspace}

\newtheorem{theorem}{Theorem}

\newtheorem{corollary}{Corollary}

\newtheorem{proposition}{Proposition}

\newcommand{\bdmath}{\begin{dmath}}
\newcommand{\edmath}{\end{dmath}}
\newcommand{\beq}{\begin{equation}}
\newcommand{\eeq}{\end{equation}}
\newcommand{\bdm}{\begin{displaymath}}
\newcommand{\edm}{\end{displaymath}}
\newcommand{\bea}{\begin{eqnarray}}
\newcommand{\eea}{\end{eqnarray}}
\newcommand{\beal}{\beq \begin{array}{lll}}
\newcommand{\eeal}{\end{array} \eeq}
\newcommand{\beas}{\begin{eqnarray*}}
\newcommand{\eeas}{\end{eqnarray*}}
\newcommand{\ba}{\begin{array}}
\newcommand{\ea}{\end{array}}
\newcommand{\bit}{\begin{itemize}}
\newcommand{\eit}{\end{itemize}}
\newcommand{\ben}{\begin{enumerate}}
\newcommand{\een}{\end{enumerate}}

\definecolor{myblue}{RGB}{65 105 225}

\newcommand{\hide}[1]{}

\newcommand{\hiddenText}{{\color{gray} hidden text.}}
\newcommand{\hideWithText}[1]{\hiddenText}

\newcommand{\morphable}{MorphQuad\xspace}

\DeclareMathOperator{\tr}{tr}

\definecolor{SpringGreen}{RGB}{40,160,80} 
\definecolor{SkyBlue}{RGB}{60,130,200}   
\definecolor{Khaki}{RGB}{160,100,70}     
\definecolor{Lavender}{RGB}{180,100,180}  

\definecolor{Teal}{RGB}{0,128,128}         
\definecolor{RoyalPurple}{RGB}{110,60,170}  
\definecolor{Bronze}{RGB}{170,110,40}       
\definecolor{Cerulean}{RGB}{0,110,200}      
\newcommand{\npjtoeditor}[1]{\textcolor{Teal}{#1}}
\newcommand{\npjtoone}[1]{\textcolor{RoyalPurple}{#1}}
\newcommand{\npjtotwo}[1]{\textcolor{Bronze}{#1}}
\newcommand{\npjtoall}[1]{\textcolor{Cerulean}{#1}}

\renewcommand{\npjtoeditor}[1]{{#1}}
\renewcommand{\npjtoone}[1]{{#1}}
\renewcommand{\npjtotwo}[1]{{#1}}
\renewcommand{\npjtoall}[1]{{#1}}

\newenvironment{sciabstract}{%
\begin{quote} \bf}
{\end{quote}}

\usepackage{caption}
\title{MorphQuad: Morphable Quadrotor for Superhuman Maneuverability, Manipulation, and Resiliency}

\author
{Jos\'{e} D\'{i}az Pe\'{o}n Gonz\'{a}lez Pacheco,$^{\star}$ Jiawei Xu,$^{\star}$ Andrew Zhao,$^{\star}$ 
Hongyu Zhou,$^{\dagger}$\\ Atharva Navsalkar,$^{\dagger}$ Andrew Scheffer,$^{\dagger}$   
Amrith Malli Reddi,$^{\dagger}$ Sashreek Shankar,$^{\dagger}$\\ Yuqing (Ivan) Bao,$^{\dagger}$
Vasileios Tzoumas\\
\normalsize{Aerospace Engineering, University of Michigan}\\
\normalsize{$^{\star}$Equal first-author contribution}~~
\normalsize{$^{\dagger}$Equal second-author contribution}\\
\normalsize{Corresponding author: Jos\'{e} D\'{i}az Pe\'{o}n Gonz\'{a}lez Pacheco (diazjose@umich.edu)}\\
}

\usepackage{booktabs}
\usepackage{colortbl}
\usepackage{xcolor}
\usepackage{pifont}
\usepackage{arydshln}  

\definecolor{headergray}{gray}{0.92}
\definecolor{scoregold}{RGB}{180,140,20}
\definecolor{scoregreen}{RGB}{30,130,76}
\definecolor{scorered}{RGB}{180,60,60}

\begin{document} 


\baselineskip24pt

\date{}
\maketitle


\begin{sciabstract}
    Infrastructure maintenance, contact-based inspection, and emergency response can benefit from aerial vehicles that act as a flying human hand with extreme maneuverability, manipulation, and resiliency (MMR): \textit{maneuverability} to fly in arbitrary orientations to reach tight and remote locations; \textit{manipulation} to point sensors, turn valves, and press tools at arbitrary orientations; \textit{resiliency} to maintain motion accuracy and force application despite disturbances from arbitrary directions, such as wind, ground effects, and friction. Realizing MMR requires more than omnidirectional flight; it also requires (I) vectoring of maximum thrust in any direction, to maximize the capacity for contact-force application and disturbance rejection, (II) almost-everywhere stability, accounting for actuation challenges, such as kinematic singularities, control singularities, and inter-rotor downwash interference, to enable reliable control over any position/orientation, and (III) miniaturizable structural designs, to enable traversal in tight spaces. Existing aerial vehicles satisfy a subset of I--III: fixed-tilt vehicles cannot vector maximum thrust in any direction due to internal thrust cancellations; variable-tilt vehicles either have kinematic and control singularities or downwash interference unaccounted for, or employ assemblies of smaller vehicles that \npjtotwo{hinder} miniaturizability.
    We present \morphable, a morphable quadrotor that delivers MMR by \npjtoall{realizing I--III on a quadrotor structure with independently articulated rotor systems through novel control and thrust allocation algorithms}. In more detail, on hardware, \morphable preserves a \npjtotwo{planar} quadrotor structure, and independently articulates each of the four rotor systems, eliminating control singularities and enabling all rotors to align and vector the platform's maximum thrust in any direction; on \npjtoall{control and thrust allocation}, \morphable enjoys \npjtoone{a closed-form null-space correction} for gimbal lock and downwash interference that maintains almost-everywhere stability via coupling with a generalized geometric controller.
    With fully-onboard computation, localization, and trajectory execution, \morphable demonstrates continuous multi-revolution rotation while translating or hovering, for applications such as pipe inspection and target tracking (maneuverability); valve turning, perching, and object pressing and pushing with human-wrist strength (manipulation); and push-pull recovery and wind rejection, even when directed to a single rotor (resiliency).

\end{sciabstract}

\section*{Introduction}
Infrastructure maintenance~\cite{park2018odar,lee2025autonomous}, contact-based inspection~\cite{bodie2020active,trujillo2019novel}, and emergency response~\cite{waharte2010supporting,alawad2023unmanned,seneviratne2018smart} can benefit from aerial vehicles that act as flying human hands with extreme maneuverability, manipulation, and resiliency (MMR):
\textbf{(a) Maneuverability}: to reach remote locations~\cite{anand2023drones,xing2024morphing}, continuously track complex trajectories in tight spaces~\cite{rs15133266,zhao2018flight,Niroui2019deep}, and inspect mega-structures such as around and in-between pipes and ducts~\cite{Ruggia2025flifo,nooralishahi2021drone} from any direction.
\textbf{(b) Manipulation}: to push obstacles away~\cite{ollero2022aerial}, turn valves~\cite{Yang2018LASDRA}, and press sensors against infrastructure surfaces to assess structural integrity~\cite{trujillo2019novel}.
\textbf{(c) Resiliency}: to move and apply forces with accuracy despite challenging disturbances, such as wind~\cite{xing2023wind,oconnell2022neuralfly}, wall, ceiling, and ground effects~\cite{yang2025ground,garofano2021aerodynamic}, and counter-acting contact forces, \textit{e.g.,} friction and elastic force~\cite{watson2025fly,alexis2016aerial,xu2025airbender}.

Realizing MMR on multirotor aerial vehicles is more than omnidirectional flight~\cite{hamandi2021design,lin2025configurations,shu2026framework}, \ie flying in arbitrary directions at any attitude; it also requires (I) vectoring of maximum thrust in any direction, (II) almost-everywhere stability, even upon accounting for actuation challenges, such as kinematic and control singularities, and inter-rotor downwash interference, and (III) miniaturizable structural designs (\Cref{fig:intro}).
\textbf{Enabler I: Vectoring of maximum thrust in any direction.} Generating maximum thrust in any desired direction ---up to the capacity of the vehicle's rotors and independently of vehicle orientation--- delivers maximum acceleration, contact force application, and disturbance rejection in any direction.
\textbf{Enabler II: Almost-everywhere stability, accounting for actuation challenges.} Maintaining stability during motion from any position and orientation to any other requires a motion control that is provably almost-everywhere stable, while accounting for practical actuation challenges such as kinematic singularity (\ie gimbal lock, where a small change in the target wrench demands rapid servo articulation, \Cref{fig:nullspace-ablation}(B)), control singularity (input/wrench-matrix rank deficiency)~\cite{bodie2018towards}, and inter-rotor downwash interference~\cite{su2022downwash}, so that the controller conforms to the hardware dynamics and guarantees reliable maneuverability and manipulation amid disturbances.
\textbf{Enabler III: Miniaturizable structural designs.} Building on a planar multirotor layout free of cages or linkages, such as a quadrotor, allows the design to shrink to a palm-sized footprint for traversing tight spaces. \npjtotwo{A planar layout removes the structural barriers to miniaturization, as demonstrated by the Crazyflie~\cite{preiss2017crazyswarm}. As with any multirotor, miniaturizing the hardware further requires compensating for lower thrust efficiency and endurance at small scales~\cite{floreano2015science,karydis2017energetics} and sourcing smaller servos that retain sufficient torque and speed.}



Among current aerial vehicles design and control, each enjoys a subset of Enablers I--III: structural and control limitations either constrain the actuation~\cite{brescianini2016design,park2016design,park2018odar,Yang2018LASDRA,tognon2018omni,hamandi2020omni,lee2024geometric,veenstra20266d,ryll2021fast,aboudorra2024modelling,ryll2012modeling,ryll2014novel,li2025six,lee2025autonomous,kamel2018voliro,ji2019modeling,skygauge,luo2026agile},
or require assemblies of smaller vehicles~\cite{su2021nullspace,zhao2018transformable,zhao2023versatile,zhao2025vectorable} or non-planar many-rotor designs~\cite{veenstra20266d}.
We next review the state of the art in more detail, from fixed-tilt to variable-tilt designs, and summarize the most relevant platforms and their hardware validation in~\Cref{tab:comparison}. A cross in the hardware-validation columns indicates the absence of experiments in the cited work.

\newpage
\begin{figure}[t!]
    \centering
    \includegraphics[width=0.95\linewidth]{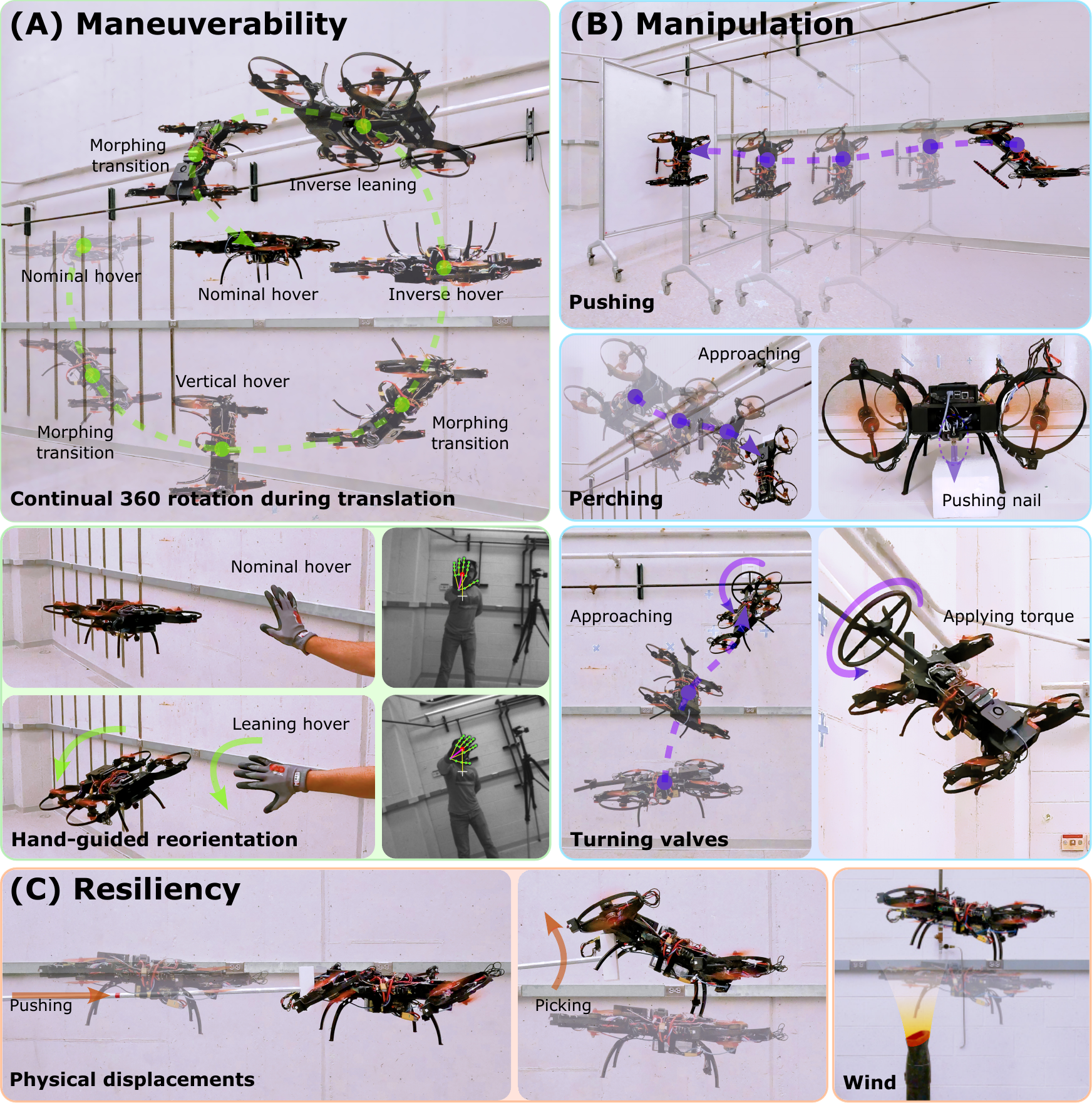}
    \caption{We present \morphable, a morphable quadrotor that promises extreme maneuverability, manipulation, and resiliency (MMR)\npjtoall{---named after the Greek word μορφή, meaning form/shape, and following similar naming conventions in the literature of thrust-vectoring multirotors~\cite{aboudorra2024modelling,ryll2021fast,tan2020morphable,acar2025mechanisms,singh2022quadplus}}. In hardware experiments with fully-onboard computation, localization, and trajectory execution without motion capture, \morphable demonstrates \textbf{(A)}~maneuverability, to perform continuous multi-revolution rotation while translating around a pipe or pointing at arbitrary directions while hovering;
    \textbf{(B)}~manipulation, for pushing objects, perching on walls and pressing nails, and turning valves at arbitrary orientations; and \textbf{(C)}~resiliency, for disturbance rejection under physical pushes, pulls, and strong winds from arbitrary directions.}
    \label{fig:intro}
\end{figure}
\clearpage

\begin{landscape}
\newpage
\begin{table}[t!]
\caption{Comparison of the most relevant omnidirectional multirotor vehicles: each vehicle and its control enjoys a subset of Enablers I--III. Specifically, their structural and control limitations either constrain the range of actuation
or require assemblies of smaller vehicles and non-planar many-rotor designs that may hinder miniaturizability.}
\label{tab:comparison}
\centering
\fontsize{9}{10}
\selectfont
\setlength{\tabcolsep}{5pt}
\renewcommand{\arraystretch}{1.0}
\renewcommand\cellset{\renewcommand\arraystretch{0.5}}
\begin{tabular}{c l c c c | c c c}
\toprule
& & \multicolumn{3}{c}{\textbf{Enablers of maneuverability, manipulation, and resiliency}} & \multicolumn{3}{c}{\textbf{Hardware validation}} \\
& &
\makecell{Vectoring of \\maximum thrust \\in any direction} &
\makecell{Almost-everywhere \\stability, accounting for \\actuation challenges} &
\makecell{Miniaturizable \\structural designs} &
\makecell{Onboard localization\\ and trajectory \\execution} &
\makecell{Manipulation \\on top of \\omnidirectional flight} & 
\makecell{Disturbance\\ rejection}\\
\midrule
\multirow{4}{*}[0.9em]{
\rotatebox[origin=c]{90}{\makecell{\textbf{Fixed-tilt}}}}
& Yang, 2018~\cite{park2016design,park2018odar,Yang2018LASDRA}  & \bc & \gc & \bc & \bc & \gc & \gc \\ \cline{2-8}
& Lee, 2024~\cite{lee2024geometric}                             & \bc & \gc & \bc & \bc & \bc & \bc \\ \cline{2-8}
& Veenstra, 2026~\cite{veenstra20266d}                          & \bc & \gc & \bc & \bc & \gc & \bc \\ 
\midrule
\multirow{10}{*}{
\rotatebox[origin=c]{90}{
\makecell{\textbf{Variable-tilt}}
}}
& Ryll, 2014~\cite{ryll2012modeling,ryll2014novel}              & \bc & \bc & \gc & \bc & \bc & \bc \\ \cline{2-8}
& Kamel, 2018~\cite{kamel2018voliro}                            & \bc & \bc & \gc & \bc & \gc & \bc \\ \cline{2-8}
& Ji, 2019~\cite{ji2019modeling}                                & \bc & \bc & \gc & \bc & \bc & \bc \\ \cline{2-8}
& Bodie, 2020~\cite{bodie2018towards,bodie2020active,bodie2019omnidirectional} 
                                                                & \bc & \bc & \gc & \bc & \gc & \gc \\ \cline{2-8}
& Su, 2021~\cite{su2021nullspace}                               & \gc & \bc & \bc & \bc & \bc & \bc \\ \cline{2-8}
& Zhao, 2023~\cite{zhao2018transformable,zhao2023versatile}     & \gc & \bc & \bc & \bc & \gc & \bc \\ \cline{2-8}
& SkyGauge, 2023~\cite{skygauge}${}^\ast$                       & \bc & \bc & \gc & $\triangle$ & $\triangle$ & $\triangle$ \\ \cline{2-8}
& Li, 2025~\cite{li2025six}                                     & \bc & \bc & \gc & \bc & \bc & \bc \\ \cline{2-8}
& Lee, 2025~\cite{lee2025autonomous}                            & \bc & \bc & \gc & \bc & \gc & \gc \\ \cline{2-8}
& Aerix, 2025~\cite{aerix2023patent,aerix2025patent}${}^\ast$   & \bc & $\triangle$ & \gc & $\triangle$ & $\triangle$ & $\triangle$ \\
\midrule
\multirow{2}{*}{\rotatebox[origin=c]{90}{\textbf{Ours}}}
& MorphEUS, 2025~\cite{iral2025morpheus}                        & \gc & \bc & \gc & \bc & \bc & \bc \\ \cline{2-8}
& \textbf{MorphQuad (this paper)}                               & \gc & \gc & \gc & \gc & \gc & \gc \\
\bottomrule
\end{tabular}
\end{table}
\noindent
${}^\ast$No peer-reviewed publications; assessed from \npjtoall{patents and public material}.

\noindent
{\smaller$\triangle$}: \npjtoall{Claimed or shown in patents and promotional videos without disclosed physical-hardware implementation, dynamics model, stability analysis, downwash treatment, control algorithms, complete details on the thrust allocation, or quantitative evaluation.}
\clearpage
\end{landscape}
\paragraph{Fixed-tilt multirotors.} Fixed-tilt multirotors do not satisfy Enabler I and cannot simultaneously realize Enablers II and III. Fixed-tilt multirotors lack rotor articulation: each rotor generates thrust in a static direction with respect to the vehicle body.
\textbf{(a) Fixed-tilt multirotors with all rotors pointing in the same direction}, such as standard quadrotors, achieve Enabler III; however, the rotors collectively generate maximum thrust only in a single direction~\cite{bouabdallah2004design,mellinger2012trajectory,michieletto2018fundamental,mueller2014stability,artale2013mathematical}.
This causes the inability to translate or apply force in lateral directions without rotating.
\textbf{(b) Fixed-tilt multirotors with rotors pointing in different directions} achieve Enabler II.
Omnidirectionality on this configuration requires a large number of rotors oriented in fixed, different directions on the vehicle body~\cite{rashad2020fully,xu2021h,falanga2018foldable,rajappa2015modeling}: seven or more unidirectional rotors~\cite{tognon2018omni,hamandi2020omni,brescianini2016design}, or six or more bidirectional rotors~\cite{park2016design,park2018odar,veenstra20266d,lee2024geometric}, all in optimized placement. Such designs require non-planar structures~\cite{tognon2018omni,hamandi2020omni,brescianini2016design}, where rotors generate counteracting thrust during operation.

\paragraph{Variable-tilt multirotors.} Existing variable-tilt multirotor vehicles have limited articulation~\cite{zheng2020tiltdrone,skygauge,luo2026agile,lee2023drone,zhu2024omnidirectional,cristobal2025gimballed}, employ linkages and assemblies of smaller vehicles~\cite{zhao2018transformable,zhao2023versatile,su2021nullspace}, or have kinematic singularity~\cite{ji2019modeling}, control singularity~\cite{ryll2012modeling,ryll2014novel,ryll2016modeling,aboudorra2024modelling,li2025six,lee2026backstepping,zhang2026learning,lee2025autonomous,kamel2018voliro}, or downwash interference~\cite{bodie2018towards,bodie2020active,bodie2019omnidirectional}  challenges.
\begin{figure}[t]
    \centering
    \includegraphics[width=1.0\textwidth]{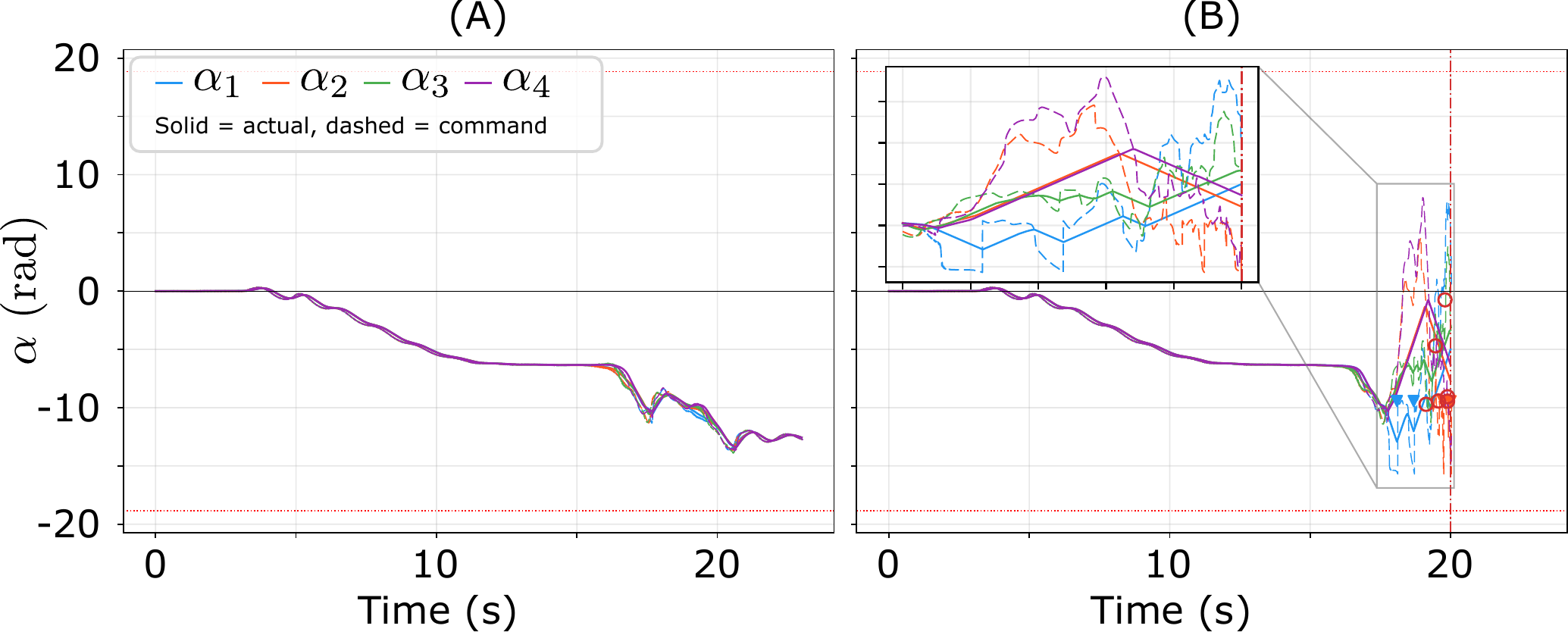}
    \caption{We demonstrate the effectiveness of our thrust allocation that accounts for the gimbal lock, using the comparison of commanded and actual outer servo positions, $\alpha_i$. (A) illustrates the servo position with the null-space correction for gimbal lock and (B) without, in a simulated trajectory tracking task. The ``vertical jumps'' of the commanded servo position (blue dashed lines in the inset plot) in (B) mark the gimbal lock, where the minimum-norm-only thrust allocation demands rapid articulation, resulting in saturated servo actuation and flight destabilization. Details of the null-space correction are available in \nameref{sec:materials}, where a similar algorithm is also used to account for downwash interference.}
    \label{fig:nullspace-ablation}
\end{figure}
\noindent\textbf{(a) Variable-tilt multirotors with limited rotor articulation} face inter-rotor thrust cancellation whenever the vehicle operates away from nominal hover, and may have a limited tilting range~\cite{zheng2020tiltdrone,skygauge,luo2026agile,cristobal2025gimballed}. Designs in this category employ synchronized rotor articulation or independent single-axis rotor articulation.
Synchronized rotor articulation~\cite{zheng2020tiltdrone,ryll2016modeling,luo2026agile,lee2023drone,zhu2024omnidirectional,aboudorra2024modelling} brings the control singularity~\cite{ryll2016modeling,aboudorra2024modelling} and can have limited thrust vectoring range due to mechanical stoppers~\cite{zheng2020tiltdrone,luo2026agile}. Independent single-axis rotor articulation that empowers a planar structure to achieve Enabler III can have control singularity and downwash interference, risking flight stability~\cite{ryll2012modeling,ryll2014novel,li2025six,lee2026backstepping,zhang2026learning,lee2025autonomous,kamel2018voliro}.
Although vehicles can mitigate the control singularity~\cite{bodie2018towards,bodie2020active,bodie2019omnidirectional} with differential thrust allocation~\cite{allenspach2020design,cuniato2026allocation}, the downwash interference remains an unresolved actuation challenge for single-axis articulation designs, on top of internal thrust cancellation.
\noindent\textbf{(b) Variable-tilt multirotors with omnidirectional rotor articulation} either employ complex interconnection structures, or may account for downwash interference but without proving almost-everywhere stability~\cite{su2022downwash,zhao2025vectorable}. Assemblies of smaller vehicles, formed by chaining~\cite{zhao2018transformable,zhao2023versatile,zhao2025vectorable} or gimbaling~\cite{su2021nullspace} multirotor modules, may not be miniaturizable. Single-vehicle designs include \npjtoall{the commercial Aerix quadrotor~\cite{aerix2023patent,aerix2025patent} and the tilting quadcopter~\cite{ji2019modeling}, a conceptual quadrotor with dual-axis gimbal rotor articulation.} Particularly, \cite{ji2019modeling} introduces only a concept quadrotor vehicle ---no physical prototype--- where each gimbal employs a single-propeller rotor system, instead of a counter-rotating pair; as a consequence, the vehicle suffers from both aerodynamic drag, when any rotor spins, and gyroscopic torque, when a rotor is tilted; further, as the controller compensates these torques even during hover, the rotors spend thrust on the torque compensations that may prevent omnidirectional maximum thrust vectoring.
\npjtoone{The Aerix quadrotor, as well as our \morphable, address the single-rotor limitations of~\cite{ji2019modeling} by employing a pair of counter-rotating rotors on each of the two-axis gimbals. 
However, Aerix quadrotor's control patent~\cite{aerix2025patent} describes a permanent offset pointing the rotors laterally outwards for safety and stability margin; the offset cancels rotor thrust persistently, thus preventing Enabler~I. 
Similarly, although Aerix quadrotor's public videos show omnidirectional maneuvers through a broad range of attitudes on physical hardware~\cite{aerix}, its control patent~\cite{aerix2025patent} does not provide the information needed ---dynamics model, downwash treatment, control algorithms, thrust allocation, stability analysis---  to establish almost-everywhere stability; the video demonstrations alone are insufficient
for this rigorous guarantee. Thus, we indicate in \Cref{tab:comparison} that the
guarantee cannot be established by the reviewed materials.}
In our preliminary work~\cite{iral2025morpheus}, we presented MorphEUS, a conceptual morphable quadrotor design that promises Enablers I--III with two-axis articulated rotors and a controller that \npjtotwo{guarantees} almost-everywhere stability with minimum-norm thrust allocation, in which extensive flight tests and actuation challenges remained unaddressed. \npjtoall{An earlier physical prototype is presented in the 2022 University of Michigan promotional video \cite{umich2022morphable}.}
In the following paragraphs, we present (i) \morphable, which enjoys an almost-everywhere stable controller, composed of a generalized geometric controller and a \npjtoall{closed-form} thrust allocation that accounts for gimbal lock and downwash interference, and (ii) the experimental evaluation of \morphable for MMR, reporting the test conditions and the tracking errors of every experiment. 


\newpage
\begin{figure}[h]
    \centering
    \includegraphics[width=1.0\linewidth]{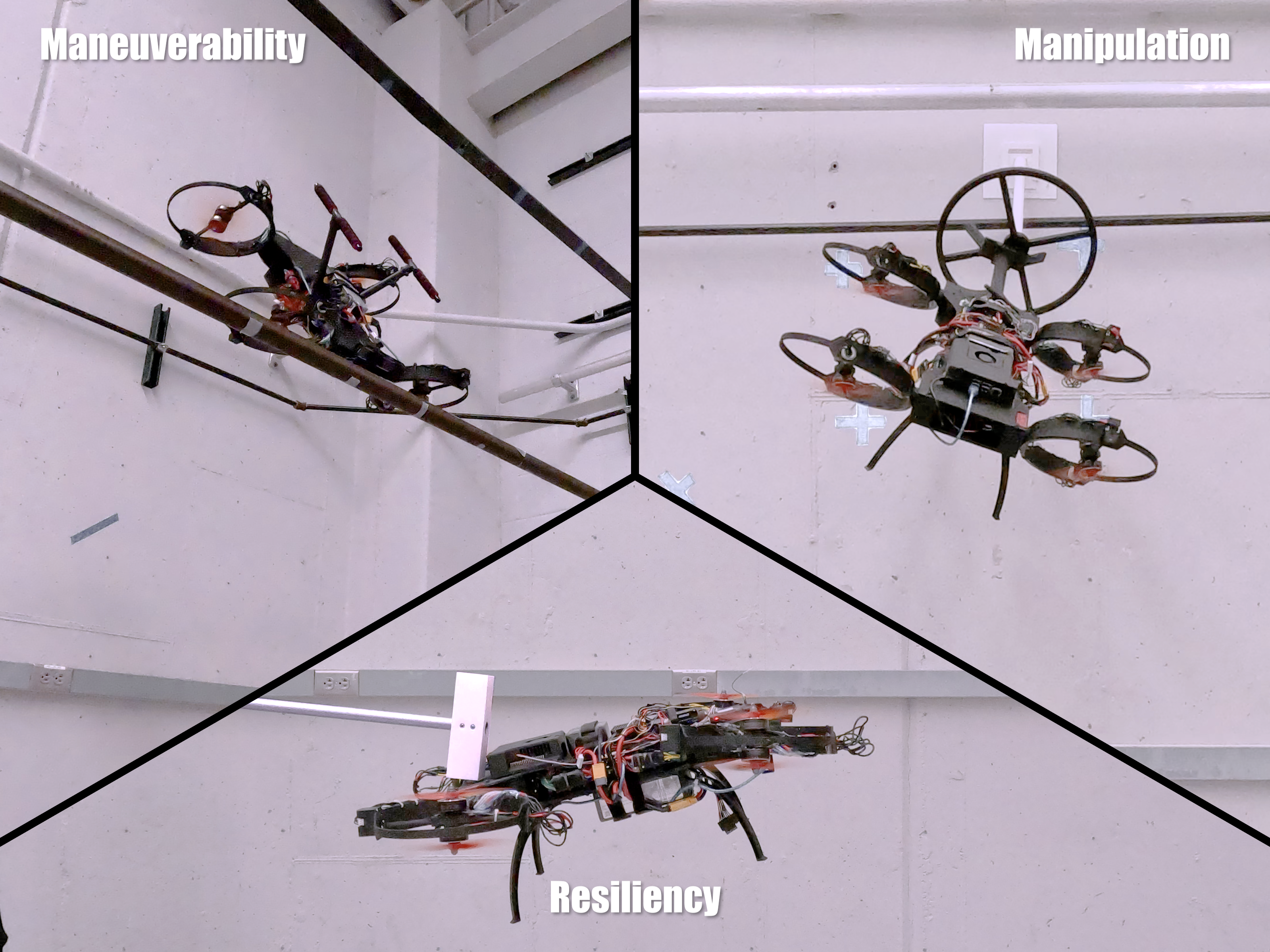}
    \caption{\npjtotwo{Screenshots from \href{https://tinyurl.com/MorphQuad-main-movie}{a movie of \morphable in action}}: We present a morphable quadrotor capable of omnidirectional free and in-contact flight ---to transform infrastructure maintenance, contact-based inspection, and emergency response. The quadrotor promises to move like a flying human hand: it can perch, push, and pull in any direction; turn inclined valves; and perform continuous helical multi-revolution motion to inspect pipes. We demonstrate fully-onboard computation, localization, and trajectory execution, with resilience to wind and push and pull disturbances.}
    \label{fig:movie_still}
\end{figure}
\clearpage

\subsection*{Contributions}
We present \morphable, a morphable multirotor that embodies Enablers I--III in a quadrotor structure, with fully-onboard computation, localization, and trajectory execution for (a)~maneuverability, (b)~manipulation, and (c)~resiliency (\Cref{fig:movie_still}). \npjtoall{To our knowledge, \morphable is the first working prototype of a quadrotor that realizes all Enablers I-III, along with a control and thrust-allocation stack with almost-everywhere stability guarantees, and extensive quantitative hardware validation that includes maneuvering, physical interaction, and disturbance rejection.}
\textbf{(a) Maneuverability:} \morphable performs continuous helical multi-revolution while pointing the camera at a pipe for inspection, and pointing while hovering for hand tracking. Existing variable-tilt vehicles \npjtoone{demonstrate} at most one full revolution~\cite{bodie2018towards,allenspach2020design,lee2025autonomous}, as further rotation requires cable ``unwinding.'' \morphable, instead, \npjtoeditor{supports up to six revolutions ($\pm1080^\circ$ from neutral), and we demonstrate $720^\circ$ of continuous rotation;} thus, \morphable can uninterruptedly execute trajectories of higher complexity ---slip rings could be used instead to enable unlimited rotation for our and existing variable-tilt vehicles.
\textbf{(b) Manipulation:} \morphable can apply torque of $4.92$\,Nm and lateral force of \npjtotwo{$44.6$\,N ($4.55$\,kgf)}, comparable to a human wrist's strength~\cite{seo2008wrist} and sufficient for industrial applications~\cite{nesbitt2011handbook}, on top of hovering. We demonstrate valve turning, wall perching and nail pressing, and whiteboard pushing. 
\textbf{(c) Resiliency:} \morphable maintains stable hover at any attitude under wind disturbances ---even when the wind is directed at a rotor system.\footnote{We apply airflow from a wind blower to the vehicle ($30$\,m/s measured at the vehicle), producing non-uniform force and torque disturbances at a small area on the vehicle; it challenges the vehicle stability, comparable to the $10$--$20$\,m/s uniform freestream generated by a wind tunnel~\cite{oconnell2022neuralfly} or fan array~\cite{huang2023datt}.} Under impulsive and sustained displacements, \morphable maintains an RMSE similar to free flight on the orthogonal axes to the disturbances.

\npjtoall{\morphable realizes I--III on a quadrotor structure with independently articulated rotor systems} through \npjtoall{novel control and thrust allocation algorithms}. In more detail, on hardware, \npjtoone{we design and construct the gimbal drivetrain and airframe, including the off-centered belt-pulley drivetrain for the gimbals} \npjtotwo{in a miniaturizable planar quadrotor layout}. \npjtoone{The gimbals, capable of multiple revolutions, independently articulate each of the four rotor systems} to maximize thrust vectoring in any direction without control singularity (\Cref{eq:compact}),
and preserve the energy-optimality of the minimum-norm thrust allocation~\cite{michael2010grasp}. 
On \npjtoall{control and thrust allocation}, 
we introduce an almost-everywhere stable controller, composed of a generalized geometric control and a thrust allocation \npjtoone{with closed-form null-space correction for gimbal lock and inter-rotor downwash interference}. To this end, we exploit the null space of the input matrix (\Cref{eq:pseudoinverse-general}) such that our allocation redistributes the rotor thrusts without altering the wrenches commanded by the geometric control \npjtoone{(\Cref{cor:wrench-preservation})}, permitting transient inter-rotor thrust cancellations to avoid gimbal lock and inter-rotor downwash interference. \npjtoone{Prior null-space allocations address gimbal singularity and inter-rotor downwash through online optimization~\cite{su2021nullspace,su2022downwash}. The counter-rotating rotor pairs of \morphable keep the input matrix constant, so we derive the null-space coefficients in closed form from the direction of the commanded force. To our knowledge, \morphable is the only variable-tilt vehicle that employs a closed-form null-space correction to preserve wrench and thus maintain the stability of the geometric controller.}
\newpage
\newpage
\begin{figure}[t!]
    \centering
    \includegraphics[width=1.0\textwidth]{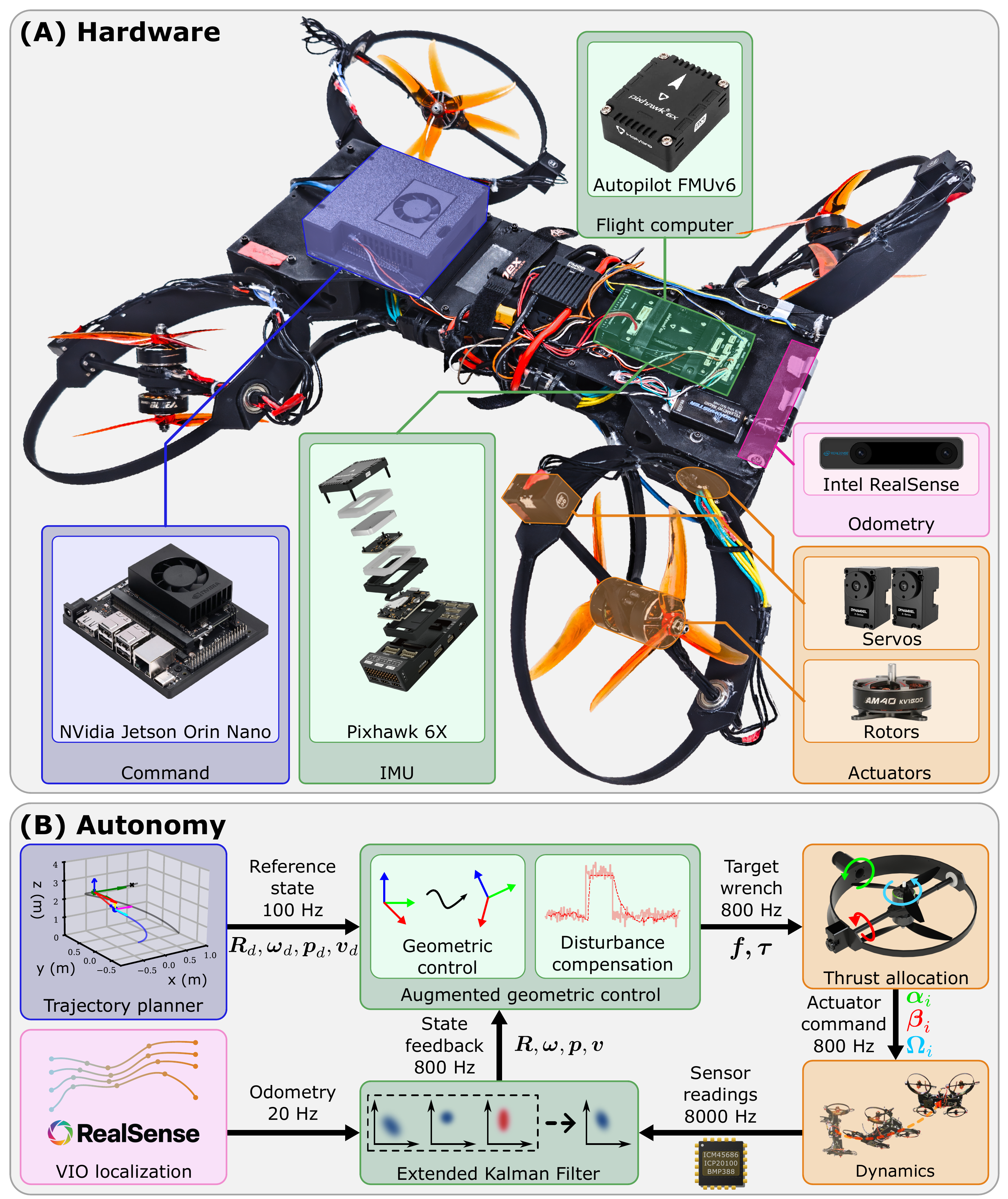}
    \caption{\textbf{(A)} Major hardware components on \morphable that enable omnidirectional flight. \textbf{(B)} The autonomy pipeline of \morphable, color-coded by its corresponding hardware in (A).}
    \label{fig:overview}
\end{figure}
\clearpage

\section*{Results}
    \label{sec:results}
    First, we discuss our contributions on hardware and autonomy that embody Enablers I--III, and then showcase the hardware experiment setup and evaluation results.

    \subsection*{Hardware and autonomy}
        \npjtoone{\Cref{fig:overview} shows the hardware (A) and the autonomy pipeline (B) of \morphable. We design and manufacture the airframe and the gimbals of \morphable, and integrate them with commercially available actuators, electronics, and the PX4 flight stack to realize Enablers I--III: the gimbals with coaxial counter-rotating rotor pairs generate maximum thrust in any direction (\textbf{Enabler I}) on a quadrotor structure (\textbf{Enabler III}), and a generalized geometric controller with a thrust allocation matched to the gimbal design provides almost-everywhere stability (\textbf{Enabler II}).
        }

        \subsubsection*{\npjtoone{Multi-revolution two-axis articulation}}
            \npjtoone{Articulating a rotor system over multiple revolutions twists its power and signal wires. This twist limits the rotation of prior tiltrotors and calls for cable unwinding~\cite{allenspach2020design}. Our gimbal routes the wires through an opening at the center of the outer rotation axis, with slack along their length to absorb the twist, and offsets the outer servo from this axis through a belt-pulley drive (\Cref{fig:actuator-stack}(A, B)).
            Each rotor system thus rotates up to $\pm 1080^\circ$ about both axes, directing its thrust anywhere on the sphere through three full revolutions from neutral (\Cref{fig:actuator-stack}(C)). Within this range, the thrust allocation selects among the equivalent servo angles the one closest to the current servo position. This selection enables the continuous multi-revolution maneuvers of our experiments.} \npjtotwo{The $\pm 1080^\circ$ range sets the maximum number of revolutions $n_r = 3$ in the thrust allocation (Materials and Methods). At this limit, we saturate the servo command instead of unwinding in flight, so we design the reference trajectories to stay within the range. All experiments stay within two revolutions.}

            \newpage
            \begin{figure}[t!]
                \centering
                \includegraphics[width=1.0\linewidth]{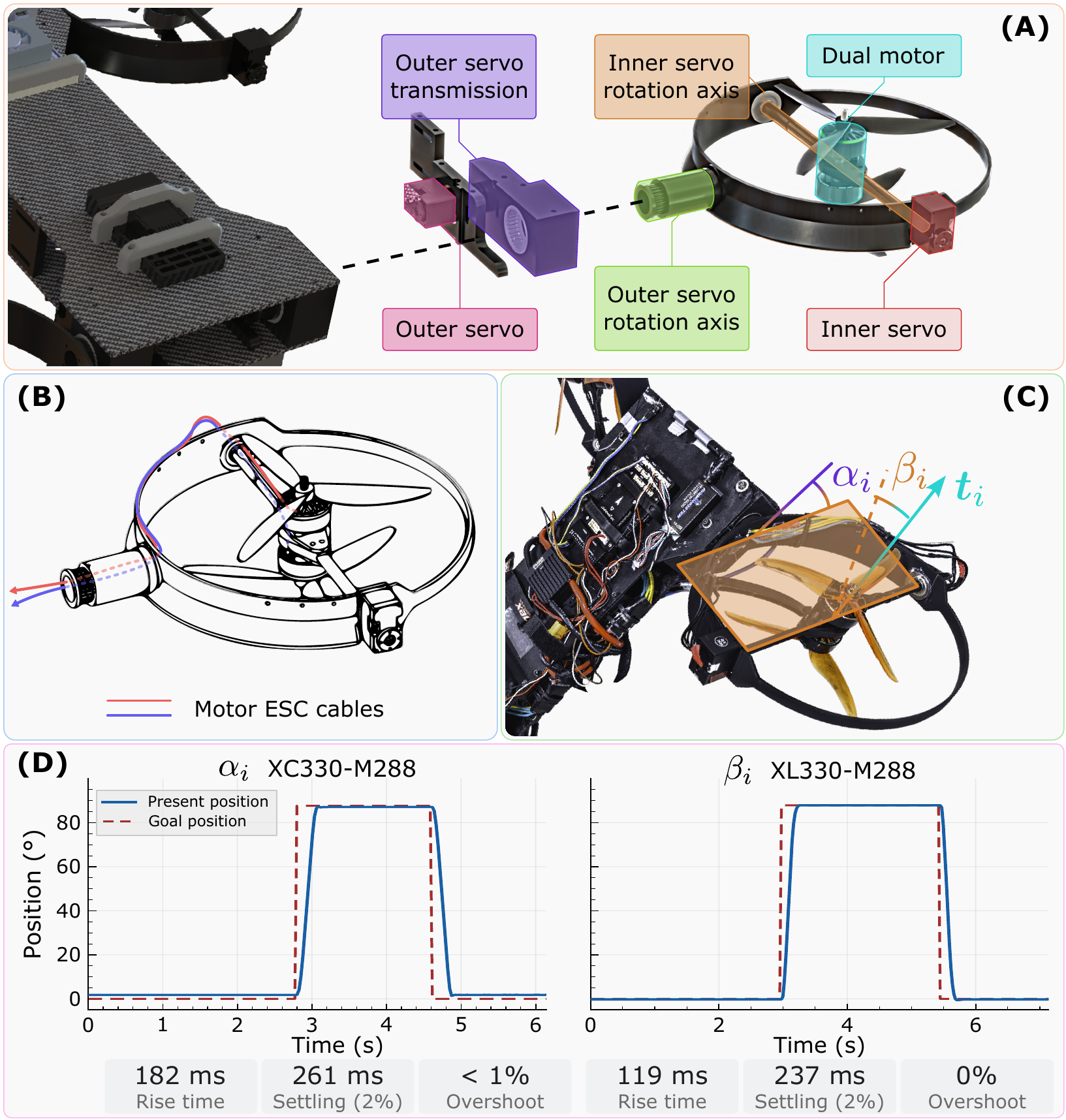}
                \caption{Each rotor system pairs two stacked motors with a two-axis gimbal, pointing the thrust in any direction while both axes spin continuously. \textbf{(A)} The inner servo rotates the motor pair inside the duct; the outer servo rotates the ring through a belt-pulley drive. \textbf{(B)} The belt-pulley drive moves the outer servo off the rotation axis, leaving the center open for wires. \textbf{(C)} Two angles describe the thrust direction: $\alpha_i$, the rotation of the duct, and $\beta_i$, the rotation of the motor pair inside the duct. \textbf{(D)} Servo response to a $90^\circ$ step command under zero-thrust load.}
                \label{fig:actuator-stack}
            \end{figure}
            \clearpage

        \subsubsection*{\npjtoone{Articulation under thrust}}
            At the intersection of the two rotation axes of each gimbal, we mount a coaxial motor pair to provide thrust. Each propeller rotates in the opposite direction at equal speeds (\Cref{fig:actuator-stack}(A)). The counter-rotation benefits the thrust vectoring and controller design by eliminating two adverse effects that arise from a single rotating propeller: (i) gyroscopic effect: a spinning rotor produces a precession torque that resists gimbal motion; and
            (ii) aerodynamic drag: a singular rotating propeller receives a reaction torque from the air along the spin axis that is inseparable from thrust generation.
            Without counter-rotation, the servo motors would need to overcome the dynamic load from (i) during thrust vectoring, and each rotor's thrust would couple with a torque from (ii). Furthermore, the aerodynamic drag would complicate \morphable's control and thrust allocation. The dual-motor solution eliminates both the gyroscopic effect and the aerodynamic drag, leading to rapid and consistent servo motion (\Cref{fig:actuator-stack}(D)) regardless of the thrust load\npjtoone{, and thus enables a low-complexity thrust allocation.} 

        \subsubsection*{\npjtoone{Allocation for gimbal lock and inter-rotor downwash}}
            \npjtoone{Our thrust allocation converts the commanded force and torque into the rotor speeds ($\Omega_i$) and gimbal angles ($\alpha_i$, $\beta_i$) through the energy-optimal minimum-norm solution. Two actuation challenges of the articulated rotor systems can destabilize this solution: gimbal lock, where one gimbal axis loses authority over the thrust direction, and inter-rotor downwash interference, where rotors direct their downwash into each other and invalidate the thrust model. Near these configurations, our allocation departs from the minimum-norm solution through the null space of the input matrix in closed form, preserving the commanded force and torque, and thus the almost-everywhere stability of the generalized geometric control. The ablation in~\Cref{fig:nullspace-ablation} shows its necessity: with the minimum-norm solution alone, \morphable loses stability during a $360^\circ$ roll.}

        \subsubsection*{\npjtoone{Fully-onboard autonomy}}
            \npjtoone{The autonomy pipeline runs fully onboard on the PX4 autopilot (\Cref{fig:overview}(B)). A quintic-polynomial planner on the companion computer streams position and attitude references. Our generalized geometric controller commands force along all three translational axes and torque about all three rotational axes, so that \morphable tracks translation and rotation independently with almost-everywhere exponential stability, in contrast to the geometric controller of regular quadrotors~\cite{lee2010geometric} that couples translation to orientation. Because \morphable achieves omnidirectional thrust vectoring, a disturbance from any direction is matched by the control input; thus, we use an $\mathcal{L}_1$ adaptive augmentation~\cite{wu2025l1quad} to estimate and compensate unknown disturbances in real time. An Extended Kalman Filter fuses visual-inertial odometry \npjtotwo{(VIO)} with inertial measurements for state feedback. We implement the controller and the allocation on the PX4 flight controller.}

    \subsection*{Experiment setup}
        We evaluate \morphable's hardware and autonomy in delivering \npjtotwo{MMR}. We design three groups of experiments (\Cref{fig:intro}): (a)~maneuverability, where \morphable performs multi-revolution continuous rotation\npjtoone{, demonstrated up to $\pm720^\circ$,} and pointing, while either translating to inspect a pipe, or hovering to track a hand movement; (b)~manipulation, where \morphable turns a valve, nails a target, and pushes an object through contact-based force application; and (c)~resiliency, where \morphable rejects winds and physical-displacement disturbances, to both the main body and individual gimbals. 
        \morphable completes all tasks with fully-onboard computation, localization, and trajectory execution, without motion capture or GPS. We next introduce the setup of each task, including objects in the environment, additional onboard equipment, and external devices. Then, we elaborate \npjtotwo{on} the evaluation metrics and \npjtotwo{provide} an overview of \morphable's performance in these experiments.
        
        \paragraph{Maneuverability experiments.}
            The pipe-inspection workspace contains two horizontal parallel pipes (\Cref{fig:maneuver}(A)) with a vertical gap of $1.3$ times \morphable's side length, forcing tight position tracking for collision-free traversal. The hand-tracking workspace has a human operator standing $1.5$\,m in front of \morphable (\Cref{fig:maneuver}(B)); the operator extends one hand toward \morphable's onboard camera, and \morphable hovers in place, tracking the hand's pointing direction and matching its roll tilt in real time while the operator moves the hand.

        \paragraph{Manipulation experiments.}
            For the valve-turning task, we mount a valve on a wall $3$\,m above the ground, with its $30$\,cm-long axis extruding outward at $45^\circ$ from the wall (\Cref{fig:manipulation}(A)); such a valve would otherwise require a human operator on a ladder to operate. We equip \morphable with a rigid ``Y''-shaped passive end-effector at its front that enters the valve's spokes to turn it. For perching and nailing, we mount a styrofoam block on a wall $2$\,m above the ground; we attach a nail perpendicularly to \morphable's bottom via a neodymium magnetic connector that releases cleanly after insertion (\Cref{fig:manipulation}(B)). For pushing, we place a \npjtotwo{$30$\,kg} wheeled whiteboard unrestrained on the floor (\Cref{fig:manipulation}(C)); \morphable needs to perch on the whiteboard to push it. 

        \paragraph{Resiliency experiments.}
            For wind disturbances, we direct a leaf blower at \morphable with its nozzle within $30$\,cm of \morphable's front face, producing an airflow of $30$\,m/s measured at the same front face (\Cref{fig:disturbance}(A)). Unlike a wind tunnel~\cite{oconnell2022neuralfly} or fan array~\cite{huang2023datt} that generates a uniform freestream, this setup concentrates peak airflow at a small incidence area, producing spatially non-uniform and time-varying force and torque disturbances that simultaneously load translation and rotation ---distinct from, not easier than, a uniform freestream. For impulsive physical displacements, a human operator wields a hammer to apply the physical displacement; the hammer head makes direct contact with \morphable's body and the individual gimbal, respectively, producing lateral pushes and off-centered picks (\Cref{fig:disturbance}(B)). For sustained disturbances, the operator applies a pulling force via a tether connected to \morphable through an inline force meter (\Cref{fig:disturbance}(C)).

        \paragraph{Task implementation and evaluation metric.}
            We define task-specific setpoints of each task for the trajectory planner and record position and orientation tracking errors throughout. Because \morphable achieves localization from onboard \npjtotwo{VIO}, we report all tracking errors in the world frame \npjtotwo{$\mathcal{W}$ (\Cref{fig:frames})} as estimated by the onboard state estimator. \npjtotwo{Against an independent Lighthouse positioning system~\cite{taffanel2021lighthouse}, the VIO position drifts by $0.5\%$ of the distance traveled in our test environment. We keep the indoor lighting consistent, since VIO is sensitive to illumination variation.} We do not evaluate the accuracy of contact force and torque since neither the trajectory planner commands contact force or torque, nor is the vehicle equipped with the necessary sensors. Specifically, for manipulation tasks, \morphable relies on position and attitude references by the planner, and maintains accurate trajectory tracking along directions orthogonal to the interaction, upon treating any counteracting forces and torques as disturbances.

        \paragraph{Evaluation summary.}
            \textbf{(a) Maneuverability}: \morphable demonstrates simultaneous rotation and translation, including $720^\circ$ continuous rotation.
            The trajectory tracking maintains a position RMSE below $6$\,cm and orientation RMSE below $5^\circ$.
            \textbf{(b) Manipulation}: \morphable demonstrates omnidirectional force and torque application, including a lateral force generation \npjtotwo{of $44.6$\,N ($4.55$\,kgf), comparable to the weight of its $4.56$\,kg takeoff mass,} and its torque application comparable to a human wrist ($4.92$\,Nm). Upon the completion of tasks, we report a position RMSE below $5$\,cm and an attitude RMSE below $3^\circ$ along axes that are not affected by the interaction.
            \textbf{(c) Resiliency}: \morphable maintains similar motion accuracy across different hovering attitudes amid winds and push/pull/pick disturbances, even when a $30$\,m/s wind gust is directed into a single rotor system and when a push lands directly on a single gimbal. \morphable achieves a position RMSE of approximately $8$\,cm along the dominant wind axis, and \npjtotwo{an} orientation RMSE below $5^\circ$. \morphable maintains its position RMSE $2$--$4$\,cm and orientation RMSE below $5^\circ$ on axes orthogonal to physical displacements.

        \newpage
        \begin{figure}[t!]
            \centering
            \includegraphics[width=1.0\linewidth]{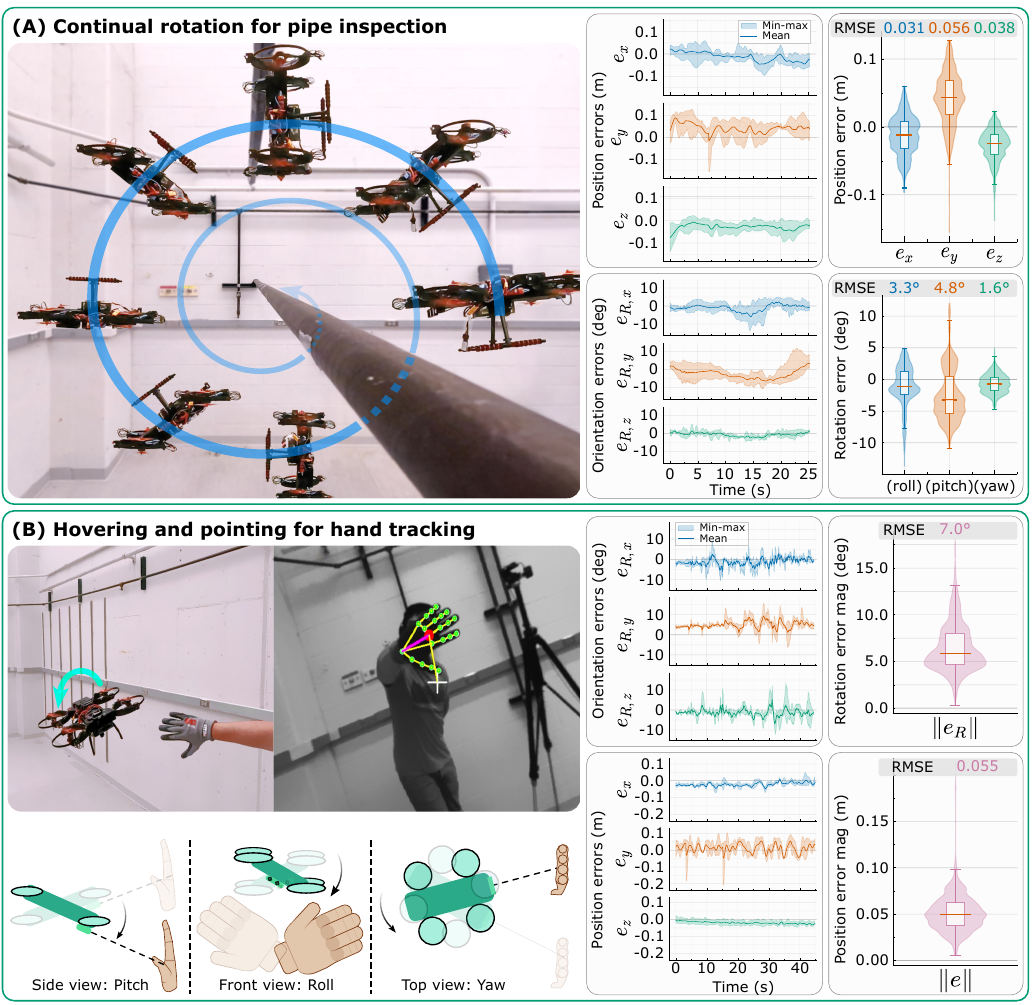}
            \caption{Maneuverability experiments. \textbf{(A)} Continuous rotation during pipe inspection:~\morphable rotates around a pipe continuously for $720^\circ$ following a circular translation trajectory, its heading direction always pointing at the pipe. \textbf{(B)}~Hovering and pointing for hand tracking: \morphable tracks a human operator's hand orientation from a fixed hover point using onboard vision, independently commanding roll, pitch, and yaw \npjtotwo{(see \Cref{fig:frames})} while maintaining position.}
            \label{fig:maneuver}
        \end{figure}
        \clearpage
        
    \subsection*{Maneuverability: Simultaneous translation and continuous multi-revolution}
        We demonstrate \morphable's motion stability and accuracy during omnidirectional flight, including simultaneous translation and continuous multi-revolution rotation. In particular, \morphable performs (a)~pipe inspection via a continuous $720^\circ$ pitch rotation around a pipe while keeping the heading direction towards the pipe; \npjtotwo{this} emulates practical maintenance applications where the robot needs to traverse between pipes and keep onboard sensors pointing towards a target and (b)~real-time hand tracking using onboard vision while hovering; \npjtotwo{this} showcases \morphable's quick and accurate response to dynamic and unpredictable target states.

        \subsubsection*{Continuous rotation for pipe inspection} 
            \morphable executes the pipe-inspection task by simultaneously translating along a circle with a $1$\,m diameter in a vertical plane around a horizontal pipe and pitching $720^\circ$ continuously, as shown in \Cref{fig:maneuver}(A). The narrow inter-pipe gap demands accurate position tracking for collision-free operation despite the excessive rotation. \morphable achieves position errors below $6$\,cm RMSE on all three axes and pitch errors at $5.5^\circ$ RMSE across the full trajectory.

        \subsubsection*{Hovering and pointing for hand tracking}
            Dynamic and unpredictable setpoints challenge the controller and hardware under conditions closer to real-world pointing tasks than pre-scripted trajectories (\Cref{fig:maneuver}(B)). We implement an onboard hand-orientation detector to estimate the operator's hand pose from the front camera feed and send orientation setpoints to the flight controller so that \morphable follows the hand pose: pitch and yaw authority allow the vehicle to point in any direction, effectively sweeping the camera over a full sphere, while roll commands keep the hand upright in the camera frame. 
        
            Despite the dynamic hand movements, \morphable maintains a position RMSE of $5.5$\,cm during hand tracking, confirming that the position controller remains decoupled from the orientation commands even under reactive setpoints. The orientation tracking RMSE is \npjtotwo{$7.0^\circ$}, with the error distribution concentrated between approximately $3^\circ$ and $8^\circ$. Occasional outliers exceeding $15^\circ$ correspond to rapid hand accelerations that momentarily saturate the visual servo reaction, accounting for the higher orientation error relative to experiment~(a).
        
    \subsection*{Manipulation: Turning valves, perching and nailing, and pushing objects}
        We demonstrate \morphable's capability of force and torque application ---at a similar level to a human wrist--- in arbitrary directions. In particular, we design three tasks of distinct objectives (\Cref{fig:manipulation}): 
        (a)~Turning a valve wheel; \npjtotwo{this} requires torque application on top of motion control.
        (b)~Perching on a wall and pressing a nail into a surface; \npjtotwo{this} requires sustained directional force application on top of motion control.
        (c)~Pushing a heavy wheeled object; \npjtotwo{this} requires sustained force application against an unrestrained moving surface, challenging the control resiliency. 
        We remark that all manipulation tasks remain open-loop with respect to contact: \morphable uses neither force/torque sensors, nor visual servoing during manipulation; instead, the trajectory planner only prescribes position and attitude setpoints. The vehicle follows the smooth trajectory that connects the setpoints upon treating all counteracting contact forces and torques as disturbances. 
        
        \subsubsection*{Valve turning}
            \morphable approaches the valve at a $45^\circ$ pitch to align the end-effector with the spokes of the valve wheel. Once engaged, \morphable locks its position and rotates about the valve axis to complete the turning, as shown in \Cref{fig:manipulation}(A). As wheeled valves are common infrastructure control mechanisms, this task highlights the potential of \morphable in maintenance automation. 
            
            \npjtotwo{Rolling to turn the wheel} is demanding since \morphable operates close to the wall: on top of the unknown torque needed to turn the valve, the aerodynamic wall effects perturb the airflow and risk the position control accuracy; consequently, a high position error would cause collision between \morphable and the surrounding structures.

            Despite the challenges, \morphable maintains position control throughout the valve turning maneuver, with position RMSE at 4\,cm along the z-axis and below 3\,cm along the x- and y-axes. Orientation RMSEs are $3.9^\circ$ in roll, $7.7^\circ$ in pitch, and $2.4^\circ$ in yaw. We attribute the elevated pitch error to the added mass of the front-mounted end-effector: the added mass induces a nose-down moment not directly accounted for in the geometric controller since we employ \morphable's nominal model without the front-mounted end-effector. We record a body torque of $4.92$\,Nm in pitch during hover\npjtotwo{, before the attitude error triggers the automatic landing. We measure this torque by pulling an off-center attachment point on \morphable via a force meter. We keep the position error low, so that the pull only loads the attitude control. Multiplying the measured force by the distance from the center of mass to the attachment point yields the vehicle torque. The roll arm is longer than the pitch arm ($b > l$, \Cref{tab:components}), so the rotor systems can deliver at least this torque in roll.} We confirm that the rotor systems deliver a torque that is comparable to human wrists~\cite{seo2008wrist} and satisfies industrial specifications~\cite{nesbitt2011handbook}.

            \subsubsection*{Perching and nailing}
                \morphable approaches the wall while continuously rotating its body, transitioning from nominal hover to a wall-parallel pose; once the nail aligns with the target, the controller applies a sustained force toward the wall to perch and press the nail into the styrofoam, as shown in \Cref{fig:manipulation}(B); finally, \morphable moves up to break the magnetic connection and retreats from the wall. This task demonstrates the application of a sustained force while holding an arbitrary pose that promises \morphable in versatile manipulation.
                
                \morphable maintains a tracking error similar to hover during the approach and rotation phase (\Cref{fig:manipulation}(B)). Upon contact, noticeable oscillation appears as the nail resists insertion. We attribute this oscillation to (i) compliance in the landing legs, which leads to vibration at impact with the wall, and (ii) the close proximity of the propellers to the wall induces turbulence that disturbs the rotor thrust. Despite these disturbances, the overall position RMSE remains approximately $5$\,cm across all three axes, with orientation RMSEs of $2.4^\circ$, $2.9^\circ$, and $2.3^\circ$ in roll, pitch, and yaw, respectively. 

            \newpage
            \begin{figure}[t!]
                \centering\includegraphics[width=1.0\textwidth]{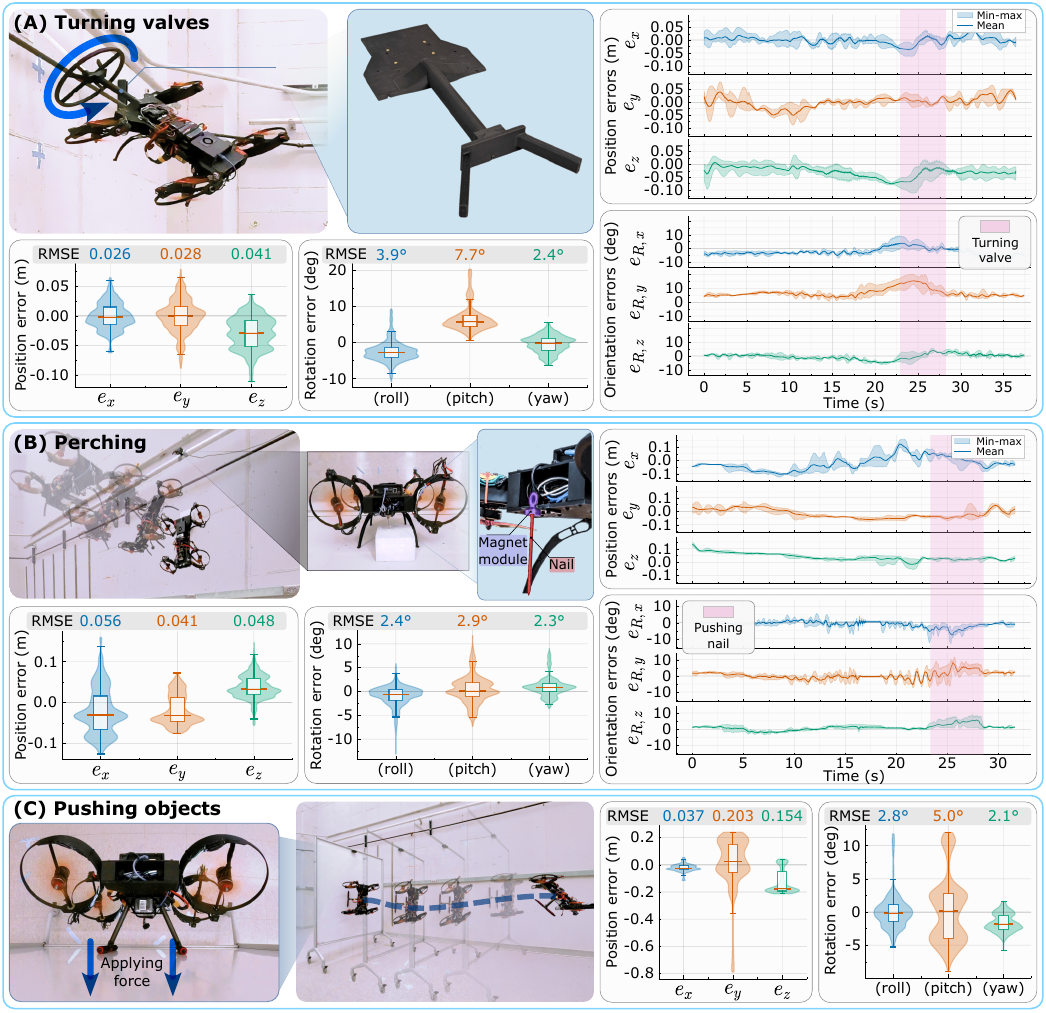}
                \caption{Manipulation experiments. \textbf{(A)}~Valve turning: \morphable approaches a wall-mounted valve and rotates the wheel by generating body torque. \textbf{(B)}~Perching and nailing: \morphable transitions from nominal hover to a wall-parallel pose and presses a magnetically attached nail into a styrofoam target. \textbf{(C)}~Pushing objects: \morphable perches onto a wheeled whiteboard and displaces it across the floor.}
                \label{fig:manipulation}
            \end{figure}
            \clearpage
        
        \subsubsection*{Pushing objects}
            We task \morphable to push the whiteboard across the floor for around $1$\,m (\Cref{fig:manipulation}(C)), emulating a door-opening or obstacle-clearing action in search-and-rescue. Unlike pressing against a rigid wall, \morphable perches onto the whiteboard \npjtotwo{and pushes it}, requiring the controller to maintain consistent contact against a surface that moves. Wall effects persist throughout the interaction. Because the manipulation pipeline currently operates without feedback on contact force, we feedforward a constant body-frame acceleration after perching to sustain the push.
            
            We report, in \Cref{fig:manipulation}(C), a noticeably higher position error along the y-axis than in other experiments. Since \npjtotwo{the feedforward push does not update} the position setpoint, the approximately $0.8$\,m of board displacement accumulates directly as reported y-axis error. The remaining axes exhibit considerably smaller errors, with x-axis RMSE at approximately $3$\,cm and z-axis RMSE at roughly $15$\,cm. We attribute the elevated z-axis error to intensified wall effects and contact friction from sustained close proximity to the whiteboard surface. The controller maintains stable attitude tracking throughout the push, with orientation RMSEs of approximately $2.8^\circ$ in roll, $5.0^\circ$ in pitch, and $2.1^\circ$ in yaw, \npjtotwo{confirming this under sustained contact}.

    \subsection*{Resiliency: Disturbance rejection under winds and physical displacements}
        We demonstrate \morphable's resilient operation under challenging disturbances ---including a $30$\,m/s ($67$\,mph or $108$\,km/h) wind stream directed into a single rotor system, applied across multiple hover attitudes.
        In particular, we design three types of disturbances that span real-world flight (\Cref{fig:disturbance}):
        (a)~Hovering \npjtotwo{at these attitudes} under sustained winds from time-varying directions; \npjtotwo{this} demonstrates \morphable's control resiliency under aerodynamic disturbances, the most common real-world stressor for aerial vehicles.
        (b)~Impulsive push and pick disturbances; \npjtotwo{this} demonstrates transient recovery from collisions and the decoupling of translational and rotational control.
        (c)~Sustained tethered pull; \npjtotwo{this} demonstrates motion accuracy under continuous load.

        \newpage
        \begin{figure}[t!]
            \centering
            \includegraphics[width=1.0\linewidth]{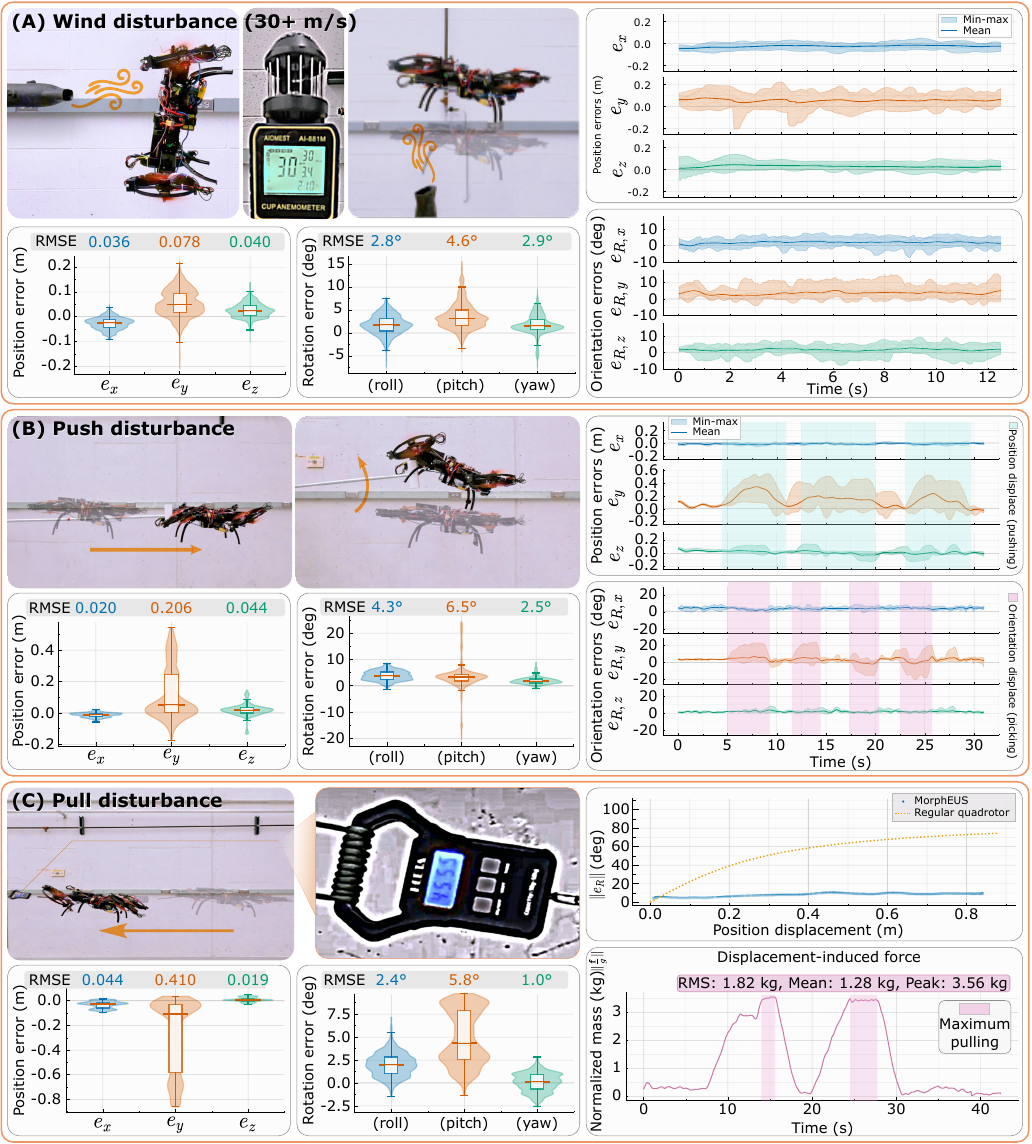}
            \caption{Resiliency experiments. \textbf{(A)}~Wind: \morphable holds a hover setpoint under time-varying leaf-blower output. \textbf{(B)}~Push and pick: \morphable endures impulsive translational and rotational displacements. \textbf{(C)}~Pull: \morphable reacts to sustained pulling force.}
            \label{fig:disturbance}
        \end{figure}
        \clearpage
        
        \subsubsection*{Strong varying wind fields}
            We design four hover scenarios that span the range of real-world wind encounters: abrupt onset from an off-axis direction, periodic sweeping across the front face of \morphable, abrupt onset while the quadrotor hovers in a vertical ($90^\circ$ pitch) pose exposing its full lateral cross-section (\Cref{fig:disturbance}(A)), and direct blowing into a single propeller creating highly asymmetric aerodynamic interference. We sustain each scenario for 30 seconds while \morphable hovers at a fixed setpoint.
            
            \morphable maintains stable hovering at different attitudes despite the winds. Because the wind is predominantly oriented along the y-axis of \morphable, the largest positional displacement appears in that direction with an RMSE of approximately 8\,cm, while the x- and z-axes remain closer to 4\,cm. The controller maintains strong attitude tracking throughout, with orientation RMSEs of $2.8^\circ$ in roll, $4.6^\circ$ in pitch, and $2.9^\circ$ in yaw. 
            These results validate the flight resiliency of \morphable under unknown, strong, and time-varying aerodynamic disturbances.

        \subsubsection*{Impulsive push and pick disturbances}
            We evaluate \morphable's ability to recover from sudden displacements by applying disturbances at irregular intervals. We apply lateral pushes to the quadrotor body that induce translational displacements, off-centered vertical pushes from below (picks) that simultaneously displace \morphable and impart a rotational disturbance, and vertical pushes to the rim of a single gimbal that load one servo directly.
            
            Push disturbances along the y-axis displace the quadrotor by as much as 50\,cm from its setpoint (\Cref{fig:disturbance}(B)). Despite these large translational excursions, the controller holds the orthogonal axes with errors comparable to maneuver tasks (x-axis RMSE: 2\,cm; z-axis: 4\,cm). Similarly, although pitch reaches over $20^\circ$ of disturbance in both directions, the controller maintains roll and yaw at $4.3^\circ$ and $2.5^\circ$ RMSE, respectively. The time-series data confirm decoupled recovery: large positional displacements do not propagate into rotational errors, and large rotational disturbances do not degrade translational tracking. 
            
        \subsubsection*{Sustained pull disturbances}
            We quantify \morphable's disturbance rejection \npjtotwo{under} sustained pulling forces. \npjtotwo{We use a force meter with a resolution of $0.005$\,N to measure the pull through a cable attached along a line through the center of mass of \morphable, as an operator gradually increases the tension.} The force sensor records a peak load of \npjtotwo{$44.6$\,N ($4.55$\,kgf)}, while the disturbance force estimated from the position error and controller state reaches \npjtotwo{$34.9$\,N ($3.56$\,kgf)}. The \npjtotwo{$9.7$\,N ($0.99$\,kgf)} discrepancy between the two measurements is attributed to online disturbance rejection that counteracts a portion of the pulling force before it manifests as position error.
            
            Consistent with the push disturbance results, a large y-axis displacement exceeding 80\,cm does not propagate into the orthogonal axes (\Cref{fig:disturbance}(C)): the x-axis RMSE remains at 4.4\,cm and the z-axis at 2\,cm, while the controller holds roll and yaw at $2.4^\circ$ and $1.0^\circ$, respectively. The pitch RMSE of $5.8^\circ$ reflects the slight upward angle of the tether pull.
            
            We contextualize this result through a scatter plot in \Cref{fig:disturbance}(C) comparing the rotational error induced by translational displacement for \morphable against a simulated standard quadrotor under identical tether loading conditions. For \morphable, increasing positional displacement produces virtually no growth in rotation error\npjtotwo{, confirming decoupling}. The simulated conventional quadrotor exhibits rapidly growing rotation error with displacement because it must pitch or roll to generate the lateral thrust needed to counteract the positional disturbance. This comparison quantitatively illustrates a fundamental advantage of \morphable's independently articulated rotor architecture: \morphable rejects external forces through independent thrust vectoring, whereas a standard quadrotor shows coupled translational and rotational responses.
            
\newpage
\section*{Discussion}

\paragraph{Summary of results.}
\morphable simultaneously achieves Enablers I--III through hardware designs and control algorithms as described in Contributions. We validate \morphable, with fully-onboard computation, localization, and trajectory execution without motion capture, in MMR experiments.
\textbf{(a) Maneuverability:} 
\morphable rotates continuously for $720^\circ$ while translating around pipes, and points while hovering to track a hand.
\textbf{(b) Manipulation:} 
\morphable generates torque of $4.92$\,Nm and lateral force of \npjtotwo{$44.6$\,N ($4.55$\,kgf)} on top of hovering ---comparable to a human wrist~\cite{seo2008wrist} and sufficient for industrial tasks~\cite{nesbitt2011handbook}--- enabling it to turn an inclined valve, perch on a wall to press a nail, and push heavy objects.
\textbf{(c) Resiliency:} \morphable holds against $30$\,m/s wind from any direction ---including wind directed into a single rotor system--- impulsive pushes to the body and to a single gimbal, and sustained pulls.
Across all tasks, \morphable holds RMSE of position within $2$--$5$\,cm and orientation within $3$--$6^\circ$ along undisturbed axes.

\paragraph{Significance.}
\morphable promises to enable autonomous infrastructure maintenance, contact-based inspection, and emergency response with a single aerial platform. \textbf{Infrastructure inspection and maintenance}: The pipe inspection 
demonstrates the simultaneous translation and continuous multi-revolution rotation needed for practical infrastructure inspection. The valve turning 
demonstrates the human-wrist-level torque needed to operate industrial mechanisms at hard-to-reach sites as both workflows today rely on scaffolding, rope access, or a human operator on a ladder. \textbf{Contact-based inspection}: The perching-and-nailing 
demonstrates the sustained directional human-wrist-level force needed to press sensors against infrastructure surfaces for integrity assessment, regardless of surface orientation. \textbf{Emergency response}: The object pushing, alongside resilience to wind, impulsive pushes, and sustained pulls,
demonstrates the obstacle-clearing capacity comparable to a human wrist; it also demonstrates disturbance robustness needed in cluttered, wind-exposed, and collapsed environments. All results are obtained \npjtoall{indoors} with fully-onboard computation, localization, and trajectory execution ---no motion capture, no GPS, no off-board computation\npjtoall{. This fully-onboard operation is a step toward} the deployment in unstructured environments. Besides, \npjtoone{our indoor experiments stress-test \morphable to the wind, contact, and impact disturbances expected during outdoor MMR tasks.}

\paragraph{Future work.} We will extend \morphable along (a) robust, adaptive, and agile flight, (b) force-aware physical interaction, and (c) natural-language control. 
\textbf{(a) Robust, adaptive, and agile flight}: \npjtoone{illumination variation and sparse visual features in outdoor, low-light, and cluttered environments challenge the current VIO, and thus require} adding \npjtoall{GPS and} LiDAR-based SLAM~\cite{zheng2024fast, ren2025safety} to \morphable's fully-onboard autonomy \npjtoone{for robust navigation}. Combining online system identification with model-predictive control~\cite{zhou2025simultaneous} would provide \morphable with adaptive transitions between free flight and in-contact flight. Explicitly modeling the dynamics of motors and servos, for example through incremental nonlinear dynamic inversion (INDI)~\cite{tal2020accurate}, could sharpen the responsiveness of thrust vectoring and push \morphable toward acrobatic-level maneuvers. Detecting and reallocating around actuator faults, such as a gimbal jammed by a collision, would extend \morphable's resiliency from external disturbances to internal failures. \npjtotwo{Lighter multi-turn servos, such as hollow-shaft direct-drive gimbal motors with integrated encoders, would let \morphable shrink to palm-sized footprints for agile flight in tighter spaces.}
\textbf{(b) Force-aware physical interaction}: \morphable's resiliency already absorbs contact forces. Adding end-effector force and torque sensing~\cite{bodie2020active} together with soft manipulators~\cite{fishman2021soft} would allow precise interaction, enabling delicate aerial tasks such as inspection, assembly, and contact-rich maintenance. \npjtoall{Shape-morphing components, such as a bi-stable perching arm~\cite{zheng2023metamorphic}, would extend \morphable's morphing beyond thrust vectoring to grasping.}
\textbf{(c) Natural-language control}: integrating vision-language-action models~\cite{tucker2026pi} would bring foundation-model manipulation ---common in fixed-base robotic arms--- into the air, letting operators direct \morphable through natural language.

\newpage
\section*{Materials and Methods}
\label{sec:materials}
We present the mechatronic components, dynamics, and control framework that enable \morphable to realize vectoring of maximum thrust in any direction and almost-everywhere stability, even accounting for gimbal lock and inter-rotor downwash interference, on a miniaturizable quadrotor design with fully-onboard computation, localization, and trajectory execution. We first identify the mechatronic components of the vehicle structure, actuation, and onboard computing. Then we derive the translational and rotational dynamics of \morphable. Finally, we present the almost-everywhere stable controller, composed of a geometric controller for singularity-free trajectory tracking on $\mathsf{SE}(3)$, and an energy-optimal thrust allocation that accounts for gimbal lock and downwash interference using the null-space of the input matrix.
We also enhance the disturbance compensation capabilities of \morphable via $\mathcal{L}_1$ adaptive augmentation. Together, these components translate high-level $\mathsf{SE}(3)$ trajectory commands into individual rotor speeds and gimbal angles that demonstrate \morphable's \npjtotwo{MMR}.

\subsection*{Mechatronic design}
\label{sec:structure}
 
\paragraph{Main body.} We preserve the compact footprint of a conventional quadrotor by arranging the four rotor systems in an H-configuration on a main body composed of two parallel carbon-fiber-infused PLA plates (\Cref{fig:overview}(A)). \npjtotwo{The rotor centers sit at $l = 0.190$\,m along $\boldsymbol{i}_{\mathcal{B}}$ and $b = 0.203$\,m along $\boldsymbol{j}_{\mathcal{B}}$ from the center of mass (\Cref{tab:components}).} The dual-plate design serves a structural and organizational role: the ESCs and power distribution electronics sit protected between the plates, the top plate houses the flight and task computers, and the bottom plate mounts the camera for \npjtotwo{VIO}. This arrangement retains the geometric profile needed to access confined and narrow environments while housing the full actuation, sensing, and computation required for omnidirectional flight. The takeoff mass, including batteries, is $4.56$\,kg, and \npjtoone{the electronics, motors, propellers, and servos} are commercially available off-the-shelf products.

\paragraph{Rotor systems.} Each rotor system pairs two coaxial counter-rotating rotors with a two-axis gimbal (\Cref{fig:actuator-stack}(A)): the counter-rotating pair cancels rotor-induced drag torque and gyroscopic precession, and the gimbal orients the pair to direct thrust in any direction. Each gimbal axis is driven by a DYNAMIXEL servo selected for its compact form factor and position feedback ---XL330-M288-T on the inner axis, XC330-M288-T on the outer--- and the eight servos daisy-chain over a single UART bus to the flight computer. Each coaxial pair consists of two brushless T-Motor VELOX 3008 1500KV motors driving 6-inch triple-blade propellers, sized for a rotor-pair thrust of \npjtotwo{$27.5$\,N ($2.8$\,kgf)} that lets \morphable hover at $40$\% throttle. Peak motor draw exceeds $43.6$\,A during fast maneuvers. This peak sets our choice of independent ESCs rated at $51$\,A continuous current.

\paragraph{Computing and simulation.} The onboard computing stack enables fully-onboard computation, localization, and trajectory execution without an external motion capture: a Pixhawk Autopilot FMUv6x flight controller runs attitude and position control, an NVIDIA Jetson Orin Nano runs trajectory generation, and an Intel RealSense \npjtotwo{T265 tracking camera} supplies the VIO sensor stream~\cite{loianno2016visual,realsense}. \npjtoone{On the Jetson, we script the time-stamped setpoints of each task and connect them with quintic-polynomial trajectories to guarantee continuous accelerations and smooth velocities~\cite{mellinger2011minimum}; the Jetson streams the resulting reference states ($\boldsymbol{R}_d$, $\boldsymbol{\omega}_d$, $\boldsymbol{p}_d$, $\boldsymbol{v}_d$) to the flight controller at $100$\,Hz. On the flight controller, an Extended Kalman Filter fuses the \npjtotwo{VIO} with the inertial measurements to provide the state feedback ($\boldsymbol{R}$, $\boldsymbol{\omega}$, $\boldsymbol{p}$, $\boldsymbol{v}$) at $800$\,Hz, and we implement the generalized geometric controller and the thrust allocation within the control and mixer modules of PX4.} \Cref{tab:components} lists the components and parameters of \morphable. To validate the full software stack prior to hardware deployment, we integrate \morphable's structural, kinematic, and dynamic models into a ROS and Gazebo software-in-the-loop (SITL) simulator.

\subsection*{Vehicle Dynamics}
\label{sec:dynamics}

The omnidirectional maximum thrust vectoring capability of \morphable originates from the gimbal-mounted coaxial rotor systems, whose servo-driven tilt angles appear directly in the equations of motion. We define four coordinate frames for the derivation: the world frame $\mathcal{W}$, the body-fixed frame $\mathcal{B}$, and a propeller-fixed frame $\mathcal{P}_i$ for each rotor arm $i \in \{1,2,3,4\}$. The frame conventions and servo angles are shown in \Cref{fig:frames}. By design, the $\mathcal{B}$-to-$\mathcal{P}_i$ transformation is parameterized by two servo positions: \npjtotwo{$\alpha_i$ controlled by the outer servo and $\beta_i$ by the inner servo}, as shown in \Cref{fig:actuator-stack}(C) and \Cref{fig:frames}. These servo angles provide additional actuator commands alongside the propeller speeds, giving \morphable a total of twelve independent control inputs.

\newpage
\begin{figure}[t!]
    \centering
    \includegraphics[width=1.0\textwidth]{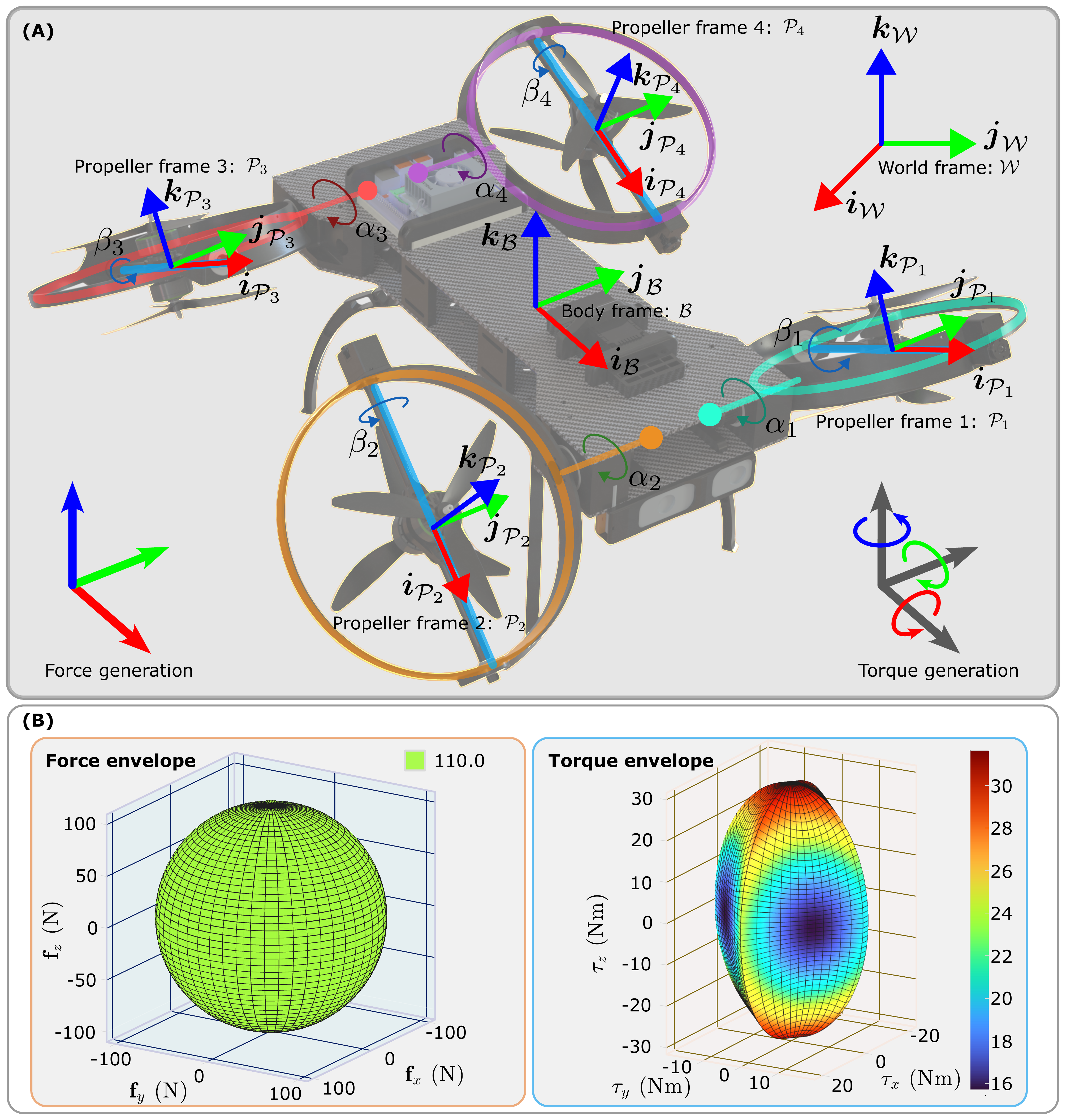}
    \caption{\textbf{(A)} Schematic of the quadrotor highlighting coordinate frames and servo angles. \textbf{(B)} Force and torque envelopes of \morphable, assuming that each arm can apply a maximum thrust of $27.5$\,N. Force envelope assumes desired torque to be zero, and vice versa. Both envelopes are obtained assuming no thrust cancellation enabled for inter-rotor downwash avoidance.}
    \label{fig:frames}
\end{figure}
\clearpage

\noindent
The translational dynamics are
\begin{equation}
    \dot{\boldsymbol{p}} = \boldsymbol{v}, \quad m \dot{\boldsymbol{v}} = m \boldsymbol{g} + \boldsymbol{R}\boldsymbol{f},
    \label{eq:translational_dynamics}
\end{equation}
where $\boldsymbol{p} \in \mathbb{R}^{3}$ and $\boldsymbol{v} \in \mathbb{R}^{3}$ are position and velocity in $\mathcal{W}$, $m$ is the quadrotor mass, $\boldsymbol{g}\in \mathbb{R}^{3}$ is the gravitational acceleration in $\mathcal{W}$, $\boldsymbol{R}\in \mathsf{SO}(3)$ is the rotation matrix from $\mathcal{B}$ to $\mathcal{W}$, and $\boldsymbol{f}\in \mathbb{R}^{3}$ is the total propeller force in $\mathcal{B}$.

The gimbal design enables the generation of rotor thrust, and thus the body forces $\boldsymbol{f}$, in any direction. Each rotor system produces a vectorized thrust whose direction is set by the servo angles $\alpha_i, \beta_i$ and whose magnitude scales with the square of propeller angular velocity $\Omega_i$. The body-frame force is
\begin{equation}
    \boldsymbol{f} = \sum_{i=1}^4 \boldsymbol{t}_i
    = c_{t} \sum_{i=1}^{4}\boldsymbol{G}_i(\alpha_i, \beta_i)\,\Omega_i^2,
    \label{eq:force_omega_relation}
\end{equation}
where $c_{t}$ is the thrust coefficient and the thrust direction vector
\begin{equation}
    \boldsymbol{G}_i(\alpha_i, \beta_i) = \begin{bmatrix}
    \sin\alpha_i \cos\beta_i  \\
    -\sin\beta_i \\
    \cos\alpha_i \cos\beta_i
\end{bmatrix}
\end{equation}
results from two intrinsic rotations of the gimbal: first by $\alpha_i$ about the $\boldsymbol{j}_{\mathcal{B}}$ axis \npjtotwo{(outer servo)}, then by $\beta_i$ about the \npjtotwo{$\boldsymbol{i}_{\mathcal{P}_i}$ axis (inner servo)}, as shown in \Cref{fig:frames}.

The rotational dynamics follow Euler's equations,
\begin{align}
    \dot{\boldsymbol{R}} = \boldsymbol{R}[\boldsymbol{\omega}]_\times, \quad \boldsymbol{J} \dot{\boldsymbol{\omega}} = -\boldsymbol{\omega} \times \boldsymbol{J} \boldsymbol{\omega} + \boldsymbol{\tau}, \label{eq:rotational_dynamics}
\end{align}
where $\boldsymbol{\omega}\in\mathbb{R}^3$ is the angular velocity in $\mathcal{B}$, $[\cdot]_\times: \mathbb{R}^3\rightarrow\mathbb{R}^{3\times 3}$ is the skew-symmetric matrix operator, $\boldsymbol{J}\in\mathbb{R}^{3\times3}$ is the inertia matrix in $\mathcal{B}$\npjtoall{, modeled as constant, independent of the gimbal angles}, and $\boldsymbol{\tau}\in\mathbb{R}^3$ is the propeller torque in $\mathcal{B}$.

A key advantage of the coaxial design is that each pair of counter-rotating propellers cancels the aerodynamic drag torque and eliminates the gyroscopic reaction that would otherwise resist gimbal motion. The servos can therefore redirect thrust using minimal effort, and the net torque on the body reduces to the moment of each thrust vector about the center of mass (COM),
\begin{align}
    \boldsymbol{\tau} = \sum_{i=1}^{4} \boldsymbol{l}_i \times \boldsymbol{t}_{i}, \label{eq:torque_thrust_relation}
\end{align}
where $\boldsymbol{l}_i = [l_{i,x},\, l_{i,y},\, l_{i,z}]^\top$ is the position vector from the COM to the center of propeller $i$, and $\boldsymbol{t}_{i} = c_t \boldsymbol{G}_i(\alpha_i, \beta_i)\,\Omega_i^2$ is the thrust of propeller $i$ in $\mathcal{B}$. Substituting yields
\begin{equation}
    \boldsymbol{\tau} = c_t \sum_{i=1}^{4} \boldsymbol{H}_i(\alpha_i, \beta_i)\,\Omega_i^2,
    \label{eq:torque_omega_relation}
\end{equation}
where
\begin{equation}
    \boldsymbol{H}_i(\alpha_i, \beta_i) = \begin{bmatrix}
    l_{i,y} \cos\alpha_i \cos\beta_i + l_{i,z} \sin\beta_i \\
    -l_{i,x} \cos\alpha_i \cos\beta_i + l_{i,z} \sin\alpha_i \cos\beta_i \\
    -l_{i,x} \sin\beta_i - l_{i,y} \sin\alpha_i \cos\beta_i
\end{bmatrix}.
\end{equation}

Combining the force~\eqref{eq:force_omega_relation} and torque~\eqref{eq:torque_omega_relation} expressions gives the full wrench,
\begin{equation}
    \boldsymbol{w} = c_t\sum_{i=1}^4\begin{bmatrix}
        \boldsymbol{G}_i(\alpha_i, \beta_i) \\ \boldsymbol{H}_i(\alpha_i, \beta_i)
    \end{bmatrix}\Omega_i^{2} = c_t \boldsymbol{F}\boldsymbol{\Omega}^{\circ2},
    \label{eq:F_matrix} 
\end{equation}
where $\boldsymbol{w} = [\boldsymbol{f}^\top \;\; \boldsymbol{\tau}^\top]^\top \in \mathbb{R}^6$ is the stacked force-torque vector, $\boldsymbol{F}(\boldsymbol{\alpha}, \boldsymbol{\beta}) \in \mathbb{R}^{6 \times 4}$ stacks the per-rotor wrench-direction vectors $[\boldsymbol{G}_i^\top \;\; \boldsymbol{H}_i^\top]^\top$ as columns, $\boldsymbol{\Omega} = [\Omega_1 \;\; \Omega_2 \;\; \Omega_3 \;\; \Omega_4]^\top$, and $(\cdot)^{\circ2}$ denotes the element-wise square.

For control design, it is convenient to treat the individual thrust vectors $\boldsymbol{t}_i \in \mathbb{R}^3$ as abstracted inputs. The thrust allocation recovers the rotor speeds $\Omega_i$ and the gimbal angles $\alpha_i$ and $\beta_i$ from these inputs. Stacking all four thrust vectors into $\boldsymbol{t} = [\boldsymbol{t}_1^\top \;\; \boldsymbol{t}_2^\top \;\; \boldsymbol{t}_3^\top \;\; \boldsymbol{t}_4^\top]^\top \in \mathbb{R}^{12}$, the wrench can be written in compact form as
\begin{equation}
    \boldsymbol{w} = \boldsymbol{M}\boldsymbol{t},
    \label{eq:compact}
\end{equation}
where $\boldsymbol{M} \in \mathbb{R}^{6 \times 12}$ is a constant input matrix determined by the fixed rotor positions $\boldsymbol{l}_i$ in $\mathcal{B}$. Since the matrix is wide and always full row-rank, the input has control over all wrench channels: the actuators can always produce a compensating input in the same subspace as an external force and torque disturbance, i.e., every disturbance is matched. In addition, $\boldsymbol{M}$ also serves as the basis for the thrust allocation described below. \npjtoone{We rigorously derive the resulting reachability of all states and realizability of every wrench in the Supplementary Materials (\Cref{prop:reachability} and \Cref{cor:controllability}).}

\subsection*{Control}\label{sec:geometric}
The central capability of \morphable is the independent generation of forces and torques in all six degrees of freedom, enabling omnidirectional trajectory tracking in $\mathsf{SE}(3)$. This section describes the three-layer control architecture: a geometric controller that tracks desired trajectories on $\mathsf{SE}(3)$~\cite{lee2010geometric}, an $\mathcal{L}_1$ adaptive augmentation that compensates for unknown disturbances~\cite{wu2025l1quad}, and a thrust allocation strategy that converts desired wrenches into rotor commands while maximizing hovering efficiency and avoiding actuation deficiencies.

\subsubsection*{Geometric controller}
The geometric controller computes desired body-frame force and torque from tracking errors defined directly on the $\mathsf{SE}(3)$ manifold, avoiding the singularities inherent in Euler angle representations. The tracking errors consist of two configuration-level errors and two rate-level errors,
\begin{equation}
    \begin{aligned}
        \boldsymbol{e_{p}} & = \boldsymbol{p} - \boldsymbol{p}_{d}, & \text{(position error)} \\
        \boldsymbol{e_R} & = \frac{1}{2}\left(\boldsymbol{R}_d^\top\boldsymbol{R} - \boldsymbol{R}^\top\boldsymbol{R}_d\right)^\vee, & \text{(orientation error)} \\[4pt]
        \boldsymbol{e_{v}} & = \boldsymbol{v} - \boldsymbol{v}_{d},  & \text{(linear velocity error)} \\
        \boldsymbol{e_{\omega}} & = \boldsymbol{\omega} - \boldsymbol{R}^{\top} \boldsymbol{R}_d \, \boldsymbol{\omega}_{d}, & \text{(angular velocity error)}
    \end{aligned}
    \label{eq:errors}
\end{equation}
where $\boldsymbol{p}_{d}$, $\boldsymbol{v}_{d}$, $\boldsymbol{R}_{d}$, $\boldsymbol{\omega}_{d}$ are the desired position, velocity, rotation matrix, and angular velocity, and $(\cdot)^\vee: \mathfrak{so}(3) \rightarrow \mathbb{R}^{3}$ is the map that extracts a 3D vector from a skew-symmetric matrix.

The desired force and torque are
\begin{equation}
    \begin{aligned}
        \boldsymbol{f}_{\mathrm{Geo}} &= \boldsymbol{R}^\top \left(-k_{\boldsymbol{p}}\boldsymbol{e_p} - k_{\boldsymbol{v}} \boldsymbol{e_v} - m\boldsymbol{g} + m\boldsymbol{\ddot p}_d\right), \\
        \boldsymbol{\tau}_{\mathrm{Geo}} &= -k_{\boldsymbol{R}} \boldsymbol{e_R} - k_{\boldsymbol{\omega}} \boldsymbol{e_\omega} + \boldsymbol{\omega} \times \boldsymbol{J}\boldsymbol{\omega} - \boldsymbol{J}\left(\boldsymbol{\omega} \times \boldsymbol{R}^\top \boldsymbol{R}_{d} \boldsymbol{\omega}_{d} - \boldsymbol{R}^\top \boldsymbol{R}_{d} \dot{\boldsymbol{\omega}}_{d}\right).
    \end{aligned}
    \label{eq:des_force_moment}
\end{equation}
\npjtotwo{Here, the positive gains $k_{\boldsymbol{p}}$ and $k_{\boldsymbol{v}}$ set the stiffness and damping of the translational error dynamics, and $k_{\boldsymbol{R}}$ and $k_{\boldsymbol{\omega}}$ those of the rotational error dynamics~\cite{lee2010geometric}. We selected initial gains in simulation and refined them in flight, tuning position and orientation separately, each as a PD controller. \Cref{tab:components} lists the gains.}
This formulation provides almost-everywhere stability, meaning convergence to any desired pose from all initial conditions except a zero-measure set~\cite{lee2010geometric}. In practice, \morphable encounters unknown disturbances such as wind and ground effects~\cite{xing2023wind,yang2025ground}. To maintain robust tracking, we augment the geometric controller with an $\mathcal{L}_1$ adaptive law that estimates and cancels these disturbances in real time.\footnote{Alternative methods include Incremental Nonlinear Dynamic Inversion~\cite{tal2020accurate}, for high-agility maneuvering, simultaneous system identification and model predictive control, for constraint-aware and mode-shifting (\eg from free flight to in-contact flight) task execution~\cite{zhou2025simultaneous}, and QP-based inner-loop control combining optimal-control formulations with INDI, for constraint-safe and robust trajectory tracking~\cite{balandi2025qp}.}

\subsubsection*{Thrust allocation}\label{sec:allocation}
As an overactuated omnidirectional platform, \morphable must convert the six-dimensional desired wrench $\boldsymbol{w} = [\boldsymbol{f}^\top \;\; \boldsymbol{\tau}^\top]^\top$ into twelve low-level commands: four gimbal angle pairs $(\alpha_i, \beta_i)$ and four rotor speeds $\Omega_i$. The allocation exploits the fact that each rotor's thrust direction and magnitude are independently controllable. Since the low-level commands $\alpha_i$, $\beta_i$, and $\Omega_i$ of each rotor are decoupled, it suffices to find the four three-dimensional thrust vectors $\boldsymbol{t}_i$; the individual commands follow from these vectors.

We decompose each rotor's thrust vector, $\boldsymbol{t}_i$, into four additive terms: one quarter of the total desired force, plus three torque contributions for roll, pitch, and yaw. For a symmetric construction with half-length $l = |l_{i,x}|$, half-width $b = |l_{i,y}|$, and diagonal distance $r = \sqrt{l^2 + b^2}$ from the COM to each propeller, the per-rotor thrust vectors are
\begin{equation}
    \begin{aligned}
        \boldsymbol{t}_1 &= \frac{\boldsymbol{f}}{4} + \frac{\tau_x}{4b}\,\boldsymbol{k}_{\mathcal{B}} - \frac{\tau_y}{4l}\,\boldsymbol{k}_{\mathcal{B}} + \frac{\tau_z}{4r}\,\hat{\boldsymbol{e}}_{\psi_1}, \\
        \boldsymbol{t}_2 &= \frac{\boldsymbol{f}}{4} - \frac{\tau_x}{4b}\,\boldsymbol{k}_{\mathcal{B}} - \frac{\tau_y}{4l}\,\boldsymbol{k}_{\mathcal{B}} + \frac{\tau_z}{4r}\,\hat{\boldsymbol{e}}_{\psi_2}, \\
        \boldsymbol{t}_3 &= \frac{\boldsymbol{f}}{4} - \frac{\tau_x}{4b}\,\boldsymbol{k}_{\mathcal{B}} + \frac{\tau_y}{4l}\,\boldsymbol{k}_{\mathcal{B}} + \frac{\tau_z}{4r}\,\hat{\boldsymbol{e}}_{\psi_3}, \\
        \boldsymbol{t}_4 &= \frac{\boldsymbol{f}}{4} + \frac{\tau_x}{4b}\,\boldsymbol{k}_{\mathcal{B}} + \frac{\tau_y}{4l}\,\boldsymbol{k}_{\mathcal{B}} + \frac{\tau_z}{4r}\,\hat{\boldsymbol{e}}_{\psi_4},
    \end{aligned}
    \label{eq:inverse}
\end{equation}
where $\boldsymbol{\tau} = [\tau_x \;\; \tau_y \;\; \tau_z]^\top$ are the scalar components of the desired torque. The unit vectors $\hat{\boldsymbol{e}}_{\psi_i}$ (denoting the yaw-torque direction for rotor $i$) lie in the $\boldsymbol{i}_{\mathcal{B}}$--$\boldsymbol{j}_{\mathcal{B}}$ plane, perpendicular to the corresponding rotor arms,
\begin{equation}
    \hat{\boldsymbol{e}}_{\psi_1} = \begin{bmatrix} -b/r \\ l/r \\ 0 \end{bmatrix},\quad
    \hat{\boldsymbol{e}}_{\psi_2} = \begin{bmatrix} b/r \\ l/r \\ 0 \end{bmatrix},\quad
    \hat{\boldsymbol{e}}_{\psi_3} = \begin{bmatrix} b/r \\ -l/r \\ 0 \end{bmatrix},\quad
    \hat{\boldsymbol{e}}_{\psi_4} = \begin{bmatrix} -b/r \\ -l/r \\ 0 \end{bmatrix}.
    \label{eq:yaw-dir}
\end{equation}
The decomposition in~\eqref{eq:inverse} is equivalent to the Moore-Penrose pseudo-inverse solution $\boldsymbol{t} = \boldsymbol{M}^{\dagger}\boldsymbol{w}$ with $\boldsymbol{M}^{\dagger}=\boldsymbol{M}^{\top}(\boldsymbol{MM}^{\top})^{-1}$. This solution yields optimal thrust efficiency, analogous to the minimum-norm allocation of a conventional quadrotor.

\begin{proposition}[Optimality]\label{prop:optimality}
The thrust distribution given by~\eqref{eq:inverse} is the optimal solution with respect to the cost function $E(\boldsymbol{t}) = \frac{1}{2}\sum_{i=1}^4 \|\boldsymbol{t}_i\|_2^2$.
\end{proposition}

Since $\boldsymbol{t}_{i} = c_t \boldsymbol{G}_i(\alpha_i, \beta_i)\,\Omega_i^2$, the rotor speed and servo angles can be recovered from each thrust vector. Defining the unit thrust direction $\hat{\boldsymbol{t}}_i = \boldsymbol{t}_i / \|\boldsymbol{t}_i\|$ with components $\hat{\boldsymbol{t}}_i = [\hat{t}_{i,x} \;\; \hat{t}_{i,y} \;\; \hat{t}_{i,z}]^\top$, the trigonometric relations give
\begin{equation}
    \Omega_i = \sqrt{\|\boldsymbol{t}_i\|/c_t}\;,\quad\text{and }
    \begin{cases}
        \beta_i = \sin^{-1}(-\hat{t}_{i,y}),\quad \alpha_i = \mathrm{atan2}\!\left(\hat{t}_{i,x},\, \hat{t}_{i,z}\right),\\[3pt]
        \beta_i = \pi-\sin^{-1}(-\hat{t}_{i,y}),\quad \alpha_i = \pi + \mathrm{atan2}\!\left(\hat{t}_{i,x},\, \hat{t}_{i,z}\right),
    \end{cases}
    \label{eq:alpha-beta-computation}
\end{equation}
where $\mathrm{atan2}(\cdot,\cdot)$ is the two-argument arctangent. The two solution branches arise because $\sin(\cdot) = \sin(\pi - \cdot)$. The output of $\sin^{-1}(\cdot)$ lies in $[-\pi/2,\, \pi/2]$, so $\beta_i$ can reach $[-\pi/2,\, 3\pi/2]$. Similarly, $\mathrm{atan2}(\cdot,\cdot)$ uses the signs of numerator and denominator to cover $[-\pi,\, \pi]$ for $\alpha_i$. To minimize servo motion at each control step, we compare the current servo position $(\alpha_i',\, \beta_i')$ with all candidate solutions $\alpha_i + 2k\pi$ and $\beta_i + 2k\pi$ for $|k| \leq n_r,\, k \in \mathbb{Z}$, and select the solution branch and revolution count $k$ that minimize $|\alpha_i' - \alpha_i - 2k\pi|$, and then, within the same branch, the revolution count that minimizes $|\beta_i' - \beta_i - 2k\pi|$, where \npjtotwo{$n_r = 3$ is the maximum number of revolutions from neutral that the wiring of the gimbals allows ($\pm 1080^\circ$; see Multi-revolution two-axis articulation in the Results).}

\subsubsection*{Null-space correction for inter-rotor downwash and gimbal lock}
We extend the minimum-norm thrust allocation~\eqref{eq:inverse} with a correction that is necessary for stable flight, especially when \morphable undergoes multi-revolution maneuvering: it accounts for (a)~inter-rotor downwash interference~\cite{su2022downwash} and (b)~gimbal lock. In more detail:
\textbf{(a)~Inter-rotor downwash interference.} When the desired torque $\boldsymbol{\tau} = \boldsymbol{0}$, the thrust allocation~\eqref{eq:inverse} yields co-linear thrust vectors $\boldsymbol{t}_i \parallel \boldsymbol{f}$ for all $i$. If the desired force is along $\boldsymbol{i}_{\mathcal{B}}$, the outer servo positions become $\alpha_i = \pm\pi/2$ with $\beta_i = 0$, so all four rotor systems point their downwash toward the opposing pair. This mutual interference invalidates the thrust model in~\eqref{eq:force_omega_relation} and reduces the effective thrust of the affected rotors.
\textbf{(b)~Gimbal lock.} If the desired force is along $\boldsymbol{j}_{\mathcal{B}}$ with $\boldsymbol{\tau} = \boldsymbol{0}$, the inner servo reaches $\beta_i = \pm\pi/2$. Then $\cos\beta_i = 0$, and the outer servo $\alpha_i$ loses authority over the thrust direction. To illustrate the practical consequence, consider \morphable hovering at a $\pi/2$ roll angle \npjtotwo{about $\boldsymbol{i}_{\mathcal{B}}$, where $\boldsymbol{j}_{\mathcal{B}}$ points upward along $\boldsymbol{k}_{\mathcal{W}}$, so that the propeller force counteracting gravity in~\eqref{eq:translational_dynamics} is} $\boldsymbol{f} = [0\;\; mg\;\; 0]^\top$. If the controller subsequently commands a small perturbation $[\epsilon\;\; mg\;\; 0]^\top$ followed by $[-\epsilon\;\; mg\;\; 0]^\top$ at high frequency, the computed $\alpha_i$ would oscillate between $\pm\pi/2$, saturating the servos and destabilizing flight. This singularity shares the same mathematical root as mapping elements of $\mathbb{S}^2$ to $\mathbb{T}^2$.

We exploit the overactuation of \morphable to implement both corrections through the null space of $\boldsymbol{M}$. \npjtoone{Prior null-space allocations address the gimbal singularity~\cite{su2021nullspace} and inter-rotor downwash~\cite{su2022downwash} through online optimization on platforms assembled from smaller vehicles; we derive the null-space coefficients in closed form.} Because $\boldsymbol{M} \in \mathbb{R}^{6\times 12}$ is a wide matrix, the linear system $\boldsymbol{w} = \boldsymbol{Mt}$ is underdetermined. The general solution takes the form
\begin{equation}
    \boldsymbol{t} = \boldsymbol{M}^{\dagger}\boldsymbol{w} + \boldsymbol{N}_M\boldsymbol{c}^N,
    \label{eq:pseudoinverse-general}
\end{equation}
where $\boldsymbol{N}_M$ is the null-space basis of $\boldsymbol{M}$ (constant rank six) and $\boldsymbol{c}^N \in \mathbb{R}^6$ is a coefficient vector that redistributes thrust among the rotors without changing the net wrench.

The coefficients are chosen to satisfy two conditions: (i) adjacent rotor pairs must not have co-linear thrust vectors when $\boldsymbol{f} \parallel \boldsymbol{i}_{\mathcal{B}}$ or $\boldsymbol{f} \parallel \boldsymbol{j}_{\mathcal{B}}$, to prevent downwash interference, and (ii) no individual thrust vector may be parallel to $\boldsymbol{j}_{\mathcal{B}}$, to avoid gimbal lock.

To satisfy condition (ii) and the second part of condition (i), we set $c^N_1 = c^N_3 = c^N_5 = \langle \boldsymbol{f}/\|\boldsymbol{f}\|,\, \boldsymbol{j}_{\mathcal{B}}\rangle$. This introduces a component that repels each inner servo position away from $\beta_i = \pm\pi/2$, proportionally to how close the desired force direction is to the gimbal-lock axis. For the first part of condition (i), we tilt the rotor pair on top away from pointing directly at the bottom pair by setting
\begin{equation}
    \begin{cases}
        c^N_2 = -c^N_4 = 0,\;\; c^N_6 = \langle\boldsymbol{f}/\|\boldsymbol{f}\|, \boldsymbol{i}_{\mathcal{B}}\rangle, & \text{if } \langle\boldsymbol{f}/\|\boldsymbol{f}\|,  \boldsymbol{i}_{\mathcal{B}}\rangle > 0, \\
        c^N_2 = -c^N_4 = \langle\boldsymbol{f}/\|\boldsymbol{f}\|, \boldsymbol{i}_{\mathcal{B}}\rangle,\;\; c^N_6 = 0, & \text{if } \langle\boldsymbol{f}/\|\boldsymbol{f}\|, \boldsymbol{i}_{\mathcal{B}}\rangle \leq 0.
    \end{cases}
    \label{eq:null-cases}
\end{equation}
This conditional assignment ensures that the anti-parallel null-space perturbation is applied to whichever rotor pair axis is most aligned with the body-frame $x$-direction, breaking the co-linearity that would otherwise cause direct downwash impingement.

\begin{proposition}[Wrench preservation]\label{cor:wrench-preservation}
For any coefficient vector $\boldsymbol{c}^N \in \mathbb{R}^6$ in~\eqref{eq:pseudoinverse-general}, including the choice of $\boldsymbol{c}^N$ above that accounts for gimbal lock and downwash interference, the thrust allocation delivers exactly the wrench commanded by the geometric controller.
\end{proposition}
\begin{proof}
Since $\boldsymbol{M}$ has constant full row rank six, $\boldsymbol{M}\boldsymbol{M}^{\dagger} = \boldsymbol{I}_6$; and since $\boldsymbol{N}_M$ is a basis of the null space of $\boldsymbol{M}$ by construction, $\boldsymbol{M}\boldsymbol{N}_M = \boldsymbol{0}$. Left-multiplying~\eqref{eq:pseudoinverse-general} by $\boldsymbol{M}$ gives
\begin{equation*}
    \boldsymbol{M}\boldsymbol{t} = \boldsymbol{M}\boldsymbol{M}^{\dagger}\boldsymbol{w} + \boldsymbol{M}\boldsymbol{N}_M\boldsymbol{c}^N = \boldsymbol{w} + \boldsymbol{0} = \boldsymbol{w},
\end{equation*}
for every $\boldsymbol{c}^N \in \mathbb{R}^6$.
\end{proof}
\noindent
Consequently, the null-space correction leaves the wrench acting on \morphable unchanged from the commanded $\boldsymbol{f}_{\mathrm{Geo}}$ and $\boldsymbol{\tau}_{\mathrm{Geo}}$ in~\eqref{eq:des_force_moment} (see \Cref{thm:exp-stable-corrected} for the resulting stability guarantee).

\npjtotwo{The null-space correction is necessary because both failure modes are fatal to flight stability.} Under~(a), the mutual downwash collapses the effective thrust of the impinged rotors, so the delivered wrench no longer matches the commanded wrench through~\eqref{eq:force_omega_relation}. Under~(b), the computed servo commands oscillate between $\pm\pi/2$, saturating the servos.
We validate the necessity of the null-space correction via simulation; the desired trajectory performs a $360^\circ$ pitch, followed by a consecutive $360^\circ$ roll, passing through the allocation singularity $\beta = \pm\pi/2$ where the outer-servo angle $\alpha$ becomes wrench-irrelevant. Without the null-space correction, the minimum-norm allocation provides no mechanism to regulate the outer servos at this singularity; the commanded $\alpha$ track winds up by more than two full turns relative to the lagging physical servos, with jumps exceeding 1\,rad, and the vehicle consequently destabilizes (\Cref{fig:nullspace-ablation}(B)). With the correction active, \morphable tracks the identical trajectory under the same controller parameters without commanded servo jumps (\Cref{fig:nullspace-ablation}(A)). This confirms that the null-space correction is a stability requirement for maneuvers crossing $90^\circ$ roll. Accordingly, all flight experiments in this work operate with the correction active.

\begin{theorem}[Exponential stability, accounting for gimbal lock and downwash interference]\label{thm:exp-stable-corrected}
Consider \morphable operating under the geometric controller~\eqref{eq:des_force_moment} together with the null-space-corrected thrust allocation~\eqref{eq:pseudoinverse-general}. Suppose the initial conditions satisfy
\begin{equation}
    \begin{gathered}
        \Psi\left(\boldsymbol{R}(0), \boldsymbol{R}_{d}(0)\right) < 2, \\
        \left\|\boldsymbol{e}_{\boldsymbol{\omega}}(0)\right\|^{2}<\frac{2 k_{\boldsymbol{R}}}{\lambda_{\max }(\boldsymbol{J})} \left(2-\Psi\left(\boldsymbol{R}(0),\boldsymbol{R}_{d}(0)\right)\right),
    \end{gathered}
    \label{eq:ROA}
\end{equation}
where $\Psi(\cdot,\cdot) : \mathsf{SO}(3) \times \mathsf{SO}(3) \rightarrow \mathbb{R}$, $\Psi\left(\boldsymbol{R}, \boldsymbol{R}_{d}\right) \triangleq \frac{1}{2} \tr\left(\boldsymbol{I} - \boldsymbol{R}_{d}^{\top} \boldsymbol{R} \right)$. Then, the zero equilibrium of the closed-loop tracking error $(\boldsymbol{e_p, e_v, e_R, e_\omega}) = (\boldsymbol{0, 0, 0, 0})$ is exponentially stable, including in flight regimes that induce gimbal lock or inter-rotor downwash interference.
\end{theorem}
\noindent See Supplementary Materials, Proof of \Cref{thm:exp-stable-corrected}, for the proof.


\bibliography{scibib}
\bibliographystyle{sciencemag}

\section*{Acknowledgments}

The authors thank the original team members Daniel Bernstein, Yanyu Chen, Laasya Chukka, and Monica Khalique, who contributed to the early development of the vehicle platform alongside A.Z. A video of the early prototype from 2022 is available at \url{https://www.youtube.com/watch?v=wFEDdiamtT8}. The project website, including supplementary videos and additional materials, is available at \url{https://iral-morphable.github.io/}.

\textbf{Funding:} This work was supported by the Michigan Translational Research and Commercialization (MTRAC) Advanced Transportation Innovation Hub, Mid-Stage Funding Program.

\textbf{Author contributions:} J.D.P.G.P.\ led the software development, including the control architecture, thrust allocation, firmware, and simulation, and conducted all experiments and data analysis. J.X.\ led the manuscript writing, visualization, and multimedia presentation, conducted all experiments, designed the thrust allocation strategy, and contributed to the controller design and augmentation. A.Z.\ led the complete electromechanical design and fabrication of the novel quadrotor from system conception to implementation and flight, and supported the integration of the controller and firmware. H.Z.\ contributed to system design, visual-inertial odometry, controller design, and stability proof of the geometric controller. A.N.\ contributed to the original thrust allocation. A.S.\ and S.S.\ led initial firmware and simulation development. A.M.R.\ designed experimental end-effectors and assisted with vehicle assembly. Y.B.\ contributed to the early vehicle design. V.T.\ led the project. All authors reviewed and edited the manuscript.

\textbf{Competing interests:} The hardware design is the subject of \npjtoone{U.S. patent application No.~19/345,590, published as US~2026/0097866~A1~\cite{zhao2026quadrotor}}. 
The almost-everywhere stable control, accounting for gimbal lock and downwash interference, is the subject of pending U.S. patent application No.~19/734,547 (filed \npjtoone{as a continuation of the previously granted U.S.~patent application No.~19/345,590}). \npjtoone{The authors have no professional, financial, institutional, or collaborative relationship with Aerix Systems, and declare no other competing interests.}

\newpage
\section*{Supplementary Materials}
\setcounter{table}{0}
\renewcommand{\thetable}{S\arabic{table}}
\subsection*{\npjtoall{Components and parameters}}
\npjtoall{\Cref{tab:components} lists the components and parameters of \morphable.}
{\small
\setlength{\tabcolsep}{4pt}
\setlength{\LTcapwidth}{\textwidth}
\begin{longtable}{>{\raggedright\arraybackslash}p{0.2\textwidth} >{\raggedright\arraybackslash}p{0.26\textwidth} >{\raggedright\arraybackslash}p{0.46\textwidth}}
\caption{\npjtoall{Components and parameters of \morphable.} \npjtotwo{Gains are listed per axis, $(x, y, z)$ for $k_{\boldsymbol{p}}, k_{\boldsymbol{v}}$ and (roll, pitch, yaw) for $k_{\boldsymbol{R}}, k_{\boldsymbol{\omega}}$.}}
\label{tab:components} \\
\toprule
\textbf{Subsystem} & \textbf{Component or parameter} & \textbf{Specification} \\
\midrule
\endfirsthead
\toprule
\textbf{Subsystem} & \textbf{Component or parameter} & \textbf{Specification} \\
\midrule
\endhead
\bottomrule
\endfoot
Main body & Frame & Two parallel carbon-fiber-infused PLA plates, H-configuration \\*
 & Takeoff mass, including batteries & $4.56$\,kg \\*
 & \npjtotwo{Inertia $\boldsymbol{J}$} & \npjtotwo{$\mathrm{diag}(0.0452,\ 0.0889,\ 0.1260)$\,kg\,m$^2$} \\*
 & Arm vectors $\boldsymbol{l}_i$ & \npjtotwo{$l = |l_{i,x}| = 0.190$\,m, $b = |l_{i,y}| = 0.203$\,m, $l_{i,z} = 0$}\\*
 & Height with landing gear & \npjtotwo{$0.31$\,m}\\
\midrule
Rotor system ($\times 4$) & Motors & Two T-Motor VELOX 3008 1500KV, counter-rotating \\*
 & Propellers & 6-inch, triple-blade \\*
 & Thrust per rotor pair & $27.5$\,N ($2.8$\,kgf); hover at $40$\% throttle \\*
 & ESCs & Independent, rated at $51$\,A continuous (peak motor draw $43.6$\,A) \\
\midrule
Gimbal ($\times 4$) & Inner-axis servo & DYNAMIXEL XL330-M288-T \\*
 & Outer-axis servo & DYNAMIXEL XC330-M288-T, belt-pulley drive \\*
 & Range per axis & $\pm 1080^\circ$ \\*
 & Servo bus & Eight servos on one UART bus to the flight controller \\
\midrule
Computing and sensing & Flight controller & Pixhawk Autopilot FMUv6x (PX4) \\*
 & Companion computer & NVIDIA Jetson Orin Nano \\*
 & Localization & Intel RealSense T265 tracking camera (VIO), fused with inertial measurements \\
\midrule
\npjtotwo{Controller gains} & \npjtotwo{Position $k_{\boldsymbol{p}}$} & \npjtotwo{$(40,\ 40,\ 80)$\,N/m} \\*
 & \npjtotwo{Velocity $k_{\boldsymbol{v}}$} & \npjtotwo{$(15,\ 15,\ 20)$\,N\,s/m} \\*
 & \npjtotwo{Orientation $k_{\boldsymbol{R}}$} & \npjtotwo{$(2.5,\ 3.2,\ 6.0)$\,N\,m/rad} \\*
 & \npjtotwo{Angular velocity $k_{\boldsymbol{\omega}}$} & \npjtotwo{$(0.70,\ 1.20,\ 1.01)$\,N\,m\,s/rad} \\
\midrule
Power & Battery & \npjtotwo{6-cell (6S) LiPo, $6000$\,mAh}\\*
 & Flight time & \npjtotwo{Approximately $5$\,min}\\
\end{longtable}
}

\subsection*{Matrix Elements}
$\boldsymbol{G}_i(\alpha_i, \beta_i)$ is the $i$-th column of $\boldsymbol{G}$, which is obtained by expanding the forward mapping from the rotor force to the body-frame force,
\begin{equation}
    \begin{aligned}
        \boldsymbol{f}&= \sum_{i=1}^{4}\boldsymbol{R}_{\mathcal{P}_i}(\alpha_i,\beta_i)\left(c_t \Omega_i^2\right){\boldsymbol{k}_{\mathcal{P}_i}}\\
        &=\sum_{i=1}^{4} \boldsymbol{R}_{{j}}(\alpha_i)\boldsymbol{R}_{{i}}(\beta_i)\left(c_t \Omega_i^2\right){\boldsymbol{k}_{\mathcal{P}_i}} \\
        &=\sum_{i=1}^{4}\left(c_t \Omega_i^2\right)\begin{bmatrix}
            \cos \alpha_i & 0 & \sin \alpha_i \\
            0 & 1 & 0 \\
            -\sin \alpha_i & 0 & \cos \alpha_i
            \end{bmatrix} \begin{bmatrix}
            1 & 0 & 0 \\
            0 & \cos \beta_i & -\sin \beta_i \\
            0 & \sin \beta_i & \cos \beta_i
        \end{bmatrix}\begin{bmatrix}
            0\\ 0\\ 1
        \end{bmatrix}\\
        &= c_{t}\begin{bmatrix}
        s\alpha_1 c\beta_1 & s\alpha_2 c\beta_2 & s\alpha_3 c\beta_3 & s\alpha_4 c\beta_4 \\
        -s\beta_1 & -s\beta_2 & -s\beta_3 & -s\beta_4 \\
        c\alpha_1 c\beta_1 & c\alpha_2 c\beta_2 & c\alpha_3 c\beta_3 & c\alpha_4 c\beta_4  \\
        \end{bmatrix}\begin{bmatrix}
        \Omega_{1}^{2} \\
        \Omega_{2}^{2} \\
        \Omega_{3}^{2} \\
        \Omega_{4}^{2}
        \end{bmatrix}, \\        
        &= c_{t} \boldsymbol{G} \boldsymbol{\Omega}^{\circ2}.
    \end{aligned}
    \label{eq:G-der}
\end{equation}

\begin{equation}
    \begin{aligned}
        \boldsymbol{\tau} &= \sum_{i=1}^{4} \boldsymbol{l}_i \times {\boldsymbol{t}}_{i},  \\
        &= \sum_{i=1}^{4} (c_t \Omega_i^2) \begin{bmatrix}
            l_{i,x} \\ l_{i,y} \\ l_{i,z} \end{bmatrix} \times \begin{bmatrix}
            c\alpha_i & s\alpha_i s\beta_i & s\alpha_i c\beta_i \\
            0 & c\beta_i & -s\beta_i \\
            -s\alpha_i & c\alpha_i s\beta_i & c\alpha_i c\beta_i
            \end{bmatrix}\begin{bmatrix}0\\0\\1\end{bmatrix},  \\
        &= \sum_{i=1}^{4} \left(c_t \Omega_i^2\right) \begin{bmatrix}
        l_{i,y} \cos \alpha_i \cos \beta_i + l_{i,z} \sin \beta_i \\
        - l_{i,x} \cos \alpha_i \cos \beta_i + l_{i,z} \sin \alpha_i \cos \beta_i \\
        - l_{i,x} \sin \beta_i - l_{i,y} \sin \alpha_i \cos \beta_i
        \end{bmatrix} \\
        &= c_t \boldsymbol{H} \boldsymbol{\Omega}^{\circ2},
    \end{aligned}
    \label{eq:H-der}
\end{equation}
where $\boldsymbol{H}^\top = \begin{bmatrix}
    l_{1,y} c\alpha_1 c\beta_1 + l_{1,z} s\beta_1 & - l_{1,x} c\alpha_1 c\beta_1 + l_{1,z} s\alpha_1 c\beta_1 & - l_{1,x} s\beta_1 - l_{1,y} s\alpha_1 c\beta_1 \\
    l_{2,y} c\alpha_2 c\beta_2 + l_{2,z} s\beta_2 & - l_{2,x} c\alpha_2 c\beta_2 + l_{2,z} s\alpha_2 c\beta_2 &  - l_{2,x} s\beta_2 - l_{2,y} s\alpha_2 c\beta_2  \\
    l_{3,y} c\alpha_3 c\beta_3 + l_{3,z} s\beta_3 & - l_{3,x} c\alpha_3 c\beta_3 + l_{3,z} s\alpha_3 c\beta_3 & - l_{3,x} s\beta_3 - l_{3,y} s\alpha_3 c\beta_3  \\
    l_{4,y} c\alpha_4 c\beta_4 + l_{4,z} s\beta_4 &    - l_{4,x} c\alpha_4 c\beta_4 + l_{4,z} s\alpha_4 c\beta_4   &   - l_{4,x} s\beta_4 - l_{4,y} s\alpha_4 c\beta_4           
\end{bmatrix}$.


\begin{equation}
    \boldsymbol{N}^\top_M = \begin{bmatrix}
        0&    0&     -1&     0&  0&   1&  0&    0&     -1&   0&    0&    1\\
        -\frac{l}{b}&    -1&    0&     \frac{l}{b}&    0&    0&    0&    0&    0&     0&   1&   0\\
        -1&    0&     0&     0&  0&   0&  0&    0&     0&    1&    0&    0\\
        -\frac{l}{b}&    -1&    0&     \frac{l}{b}&    0&    0&    0&    1&   0&    0&   0&     0\\
        0&    0&     0&     -1&  0&   0&  1&    0&     0&    0&    0&    0\\
        0&    -1&     0&     0&  1&   0&  0&    0&     0&    0&    0&    0
    \end{bmatrix}.
\end{equation}

\subsection*{\npjtoone{Reachability and controllability}}
\begin{proposition}[Reachability]\label{prop:reachability}
Given the vehicle dynamics~\eqref{eq:translational_dynamics},\eqref{eq:rotational_dynamics}, the force and torque relation~\eqref{eq:F_matrix}, and sufficient limits on rotor velocities, all states $(\boldsymbol{p}, \boldsymbol{v}, \boldsymbol{R}, \boldsymbol{\omega}) \in \mathbb{R}^3 \times \mathbb{R}^3 \times \mathsf{SO}(3) \times \mathbb{R}^3$ (position, linear velocity, rotation matrix, angular velocity) are reachable.
\end{proposition}

\begin{corollary}[Controllability]\label{cor:controllability}
For any desired force $\boldsymbol{f}$ and torque $\boldsymbol{\tau}$, there exists a set of rotor thrust vectors $\boldsymbol{t} = [\boldsymbol{t}_1^\top \; \boldsymbol{t}_2^\top \; \boldsymbol{t}_3^\top \; \boldsymbol{t}_4^\top]^\top$ such that $\boldsymbol{f} = \sum_{i=1}^4 \boldsymbol{t}_i$ and $\boldsymbol{\tau} = \sum_{i=1}^4 \boldsymbol{l}_i \times \boldsymbol{t}_i$; that is, every admissible force and torque is realizable by the actuators.
\end{corollary}
\begin{proof}
Since $\boldsymbol{M}$ is full row rank, it is surjective onto $\mathbb{R}^6$: for any $\boldsymbol{w} = [\boldsymbol{f}^\top \; \boldsymbol{\tau}^\top]^\top$, the linear system $\boldsymbol{w} = \boldsymbol{M}\boldsymbol{t}$ admits at least one solution $\boldsymbol{t}$.
\end{proof}

While this can be intuitively understood based on design, we rigorously prove \Cref{prop:reachability}: that the system is reachable with respect to all position and orientation states and their derivatives.

\begin{proof}[Proof of \Cref{prop:reachability}]
By design, thrust vector by $i^{th}$ rotor, $\boldsymbol{t}_i$ is given by
\begin{equation}
    \boldsymbol{t}_i =   c_t\Omega_i^2
    \begin{bmatrix}
        \sin\alpha_i \cos\beta_i \\
        -\sin\beta_i \\ \cos\alpha_i\cos\beta_i
    \end{bmatrix}, \label{eq:ti}
\end{equation}
where $\alpha_i$ and $\beta_i$ are the servo motor angles, $\Omega_i$ is the angular velocity of $i^{th}$ rotor and $c_t$ is the thrust coefficient. Based on relation \eqref{eq:ti}, 
we can always find a solution $(\alpha_i, \beta_i, \Omega_i)$ for any given $\boldsymbol{t}_i$, using sine and cosine inverses. 
In other words, any vector $\boldsymbol{t}_i$ can be achieved; hence, we consider the combined thrust vector $\boldsymbol{t} = [\boldsymbol{t}_1^\top \; \boldsymbol{t}_2^\top \; \boldsymbol{t}_3^\top \; \boldsymbol{t}_4^\top]^\top$, as a high-level control input. 

\textbf{\textit{Note:}} Only the magnitude of $\boldsymbol{t}_i$ has an upper limit, solely based on the maximum possible angular speed of propeller $\Omega_i$. Any direction, based on servo angles $\alpha_i$ and $\beta_i$, can be commanded. 

The forces and torques on the drone due to each of the propeller thrusts are given by 
\begin{equation}
    \begin{bmatrix}
        \boldsymbol{f} \\ \boldsymbol{\tau}
    \end{bmatrix} = \begin{bmatrix}
        \sum_{i=1}^4 \boldsymbol{t}_i \\ \sum_{i=1}^4 \boldsymbol{l}_i\times\boldsymbol{t}_i 
    \end{bmatrix} = \begin{bmatrix} \boldsymbol{I} & \boldsymbol{I} & \boldsymbol{I} & \boldsymbol{I} \\ \boldsymbol{l}_1 {\times} & \boldsymbol{l}_2 {\times} & \boldsymbol{l}_3 {\times} & \boldsymbol{l}_4{\times} \end{bmatrix} \begin{bmatrix}
        \boldsymbol{t}_1\\\boldsymbol{t}_2\\\boldsymbol{t}_3\\\boldsymbol{t}_4
    \end{bmatrix} = \boldsymbol{M} \begin{bmatrix}
        \boldsymbol{t}_1\\\boldsymbol{t}_2\\\boldsymbol{t}_3\\\boldsymbol{t}_4
    \end{bmatrix}, \label{eq:linear_force_torque}
\end{equation}
where $\boldsymbol{M}  = \begin{bmatrix}
    \scriptstyle 1 & \scriptstyle 0 & \scriptstyle 0 & \scriptstyle 1 & \scriptstyle 0 & \scriptstyle 0 & \scriptstyle 1 & \scriptstyle 0 & \scriptstyle 0 & \scriptstyle 1 & \scriptstyle 0 & \scriptstyle 0 \\
    \scriptstyle 0 & \scriptstyle 1 & \scriptstyle 0 & \scriptstyle 0 & \scriptstyle 1 & \scriptstyle 0 & \scriptstyle 0 & \scriptstyle 1 & \scriptstyle 0 & \scriptstyle 0 & \scriptstyle 1 & \scriptstyle 0 \\
    \scriptstyle 0 & \scriptstyle 0 & \scriptstyle 1 & \scriptstyle 0 & \scriptstyle 0 & \scriptstyle 1 & \scriptstyle 0 & \scriptstyle 0 & \scriptstyle 1 & \scriptstyle 0 & \scriptstyle 0 & \scriptstyle 1 \\
    \scriptstyle 0 & \scriptstyle 0 & \scriptstyle b & \scriptstyle 0 & \scriptstyle 0 & \scriptstyle -b & \scriptstyle 0 & \scriptstyle 0 & \scriptstyle -b & \scriptstyle 0 & \scriptstyle 0 & \scriptstyle b \\
    \scriptstyle 0 & \scriptstyle 0 & \scriptstyle -l & \scriptstyle 0 & \scriptstyle 0 & \scriptstyle -l & \scriptstyle 0 & \scriptstyle 0 & \scriptstyle l & \scriptstyle 0 & \scriptstyle 0 & \scriptstyle l \\
    \scriptstyle -b & \scriptstyle l & \scriptstyle 0 & \scriptstyle b & \scriptstyle l & \scriptstyle 0 & \scriptstyle b & \scriptstyle -l & \scriptstyle 0 & \scriptstyle -b & \scriptstyle -l & \scriptstyle 0 
\end{bmatrix}$, and $l$ and $b$ are the half-length and half-breadth of the drone, respectively, as shown in \Cref{eq:inverse}. The linear mapping matrix $\boldsymbol{M}:\boldsymbol{t}\rightarrow (\boldsymbol{f},\boldsymbol{\tau})$ is an underdetermined system since $\boldsymbol{M}$ is a $6\times12$ matrix with rank 6. The columns of $M$ therefore span the vector space $\mathbb{R}^6$. Thus, this system of equations has infinite solutions. Since we always have a solution, we prove that any force/torque is achievable. 
Given the dynamics in \cref{eq:translational_dynamics,eq:rotational_dynamics}, \ie
\begin{equation*}
    \begin{aligned}
        \dot{\boldsymbol{p}} &= \boldsymbol{v}, \quad \quad \;\;\;\; m \dot{\boldsymbol{v}} = m \boldsymbol{g} + \boldsymbol{R}\boldsymbol{f}, \\
        \dot{\boldsymbol{R}} &= \boldsymbol{R}[{\boldsymbol{\omega}}]_\times, \quad  \boldsymbol{J} \dot{\boldsymbol{\omega}} = -\boldsymbol{\omega} \times \boldsymbol{J} \boldsymbol{\omega} + \boldsymbol{\tau},
    \end{aligned}    
\end{equation*}
and that the maximum limit on $\Omega_i$ is reasonable, so that the force/torque are always enough to counter gravity $\boldsymbol{g}$ and gyroscopic moments $-\boldsymbol{\omega} \times \boldsymbol{J} \boldsymbol{\omega}$, we prove that all states are \textit{reachable}.
\end{proof}

\subsection*{Proof of \Cref{thm:exp-stable-corrected}}
Using the controller defined in \eqref{eq:des_force_moment}, the corresponding closed loop error dynamics is
\begin{align}
    \boldsymbol{\dot{e}_p} &= \boldsymbol{e_v} \\
    \boldsymbol{\dot{e}_v} &= \frac{1}{m}\left(-k_{\boldsymbol{p}}\boldsymbol{e_p} -  k_{\boldsymbol{v}}\boldsymbol{e_v}\right)\\
    \boldsymbol{\dot{e}_R} &= \frac{1}{2} \left(tr\left(\boldsymbol{R}^{\top} \boldsymbol{R}_{d}\right)\boldsymbol{I} - \boldsymbol{R}^{\top} \boldsymbol{R}_{d}\right)\boldsymbol{e_\omega} \\
    \boldsymbol{\dot{e}_\omega} &= \boldsymbol{J}^{-1}\left(-k_{\boldsymbol{R}}\boldsymbol{e_R} - k_{\boldsymbol{\omega}} \boldsymbol{e_\omega} \right)
\end{align}
The error state $(\boldsymbol{e_p, e_v, e_R, e_\omega}) = (\boldsymbol{0, 0, 0, 0})$ is an exponentially stable equilibrium of the closed loop dynamics. We define the regions of attraction and prove \Cref{thm:exp-stable-corrected} below.

\begin{proof}[Proof of \Cref{thm:exp-stable-corrected}]
By \Cref{cor:wrench-preservation}, the thrust allocation delivers exactly the commanded wrench $\boldsymbol{f}_{\mathrm{Geo}}, \boldsymbol{\tau}_{\mathrm{Geo}}$ regardless of the null-space correction; hence the closed-loop error dynamics under the corrected allocation are identical to those obtained by substituting the controller~\eqref{eq:des_force_moment} directly into the vehicle dynamics~\eqref{eq:translational_dynamics},\eqref{eq:rotational_dynamics}, as derived above.

We define the Lyapunov Candidate for the attitude dynamics. Let a Lyapunov candidate $\mathcal{V}_{2}$ be
\begin{equation}
    \mathcal{V}_{2}=\frac{1}{2} \boldsymbol{e}_{\boldsymbol{\omega}}^\top \boldsymbol{J} \boldsymbol{e}_{\boldsymbol{\omega}} + {k}_{\boldsymbol{R}} \Psi\left(\boldsymbol{R}, \boldsymbol{R}_{d}\right) + c_{2} \boldsymbol{e}_{\boldsymbol{R}}^\top \boldsymbol{e}_{\boldsymbol{\omega}},
\end{equation}
where $c_{2}$ is a non-negative constant. Suppose the initial condition in \cref{eq:ROA} is satisfied, it follows directly \cite[Proposition~1]{lee2010geometric} that 
\begin{align}
    z_{2}^{\top} M_{21} z_{2} \leq \mathcal{V}_{2} \leq z_{2}^{\top} M_{22} z_{2} \\
    \dot{\mathcal{V}}_{2} \leq -z_{2}^{\top} W_{2} z_{2}
\end{align}
where $z_{2}=\left[\begin{array}{c}\left\|\boldsymbol{e}_{\boldsymbol{R}}\right\| \\
\left\|\boldsymbol{e}_{\boldsymbol{\omega}}\right\|\end{array}\right]$,
$M_{21}=\frac{1}{2}\left[\begin{array}{cc}
    k_{\boldsymbol{R}} & -c_{2} \\ 
    -c_{2} & \lambda_{\min }(\boldsymbol{J}) 
        \end{array}\right]$,
$M_{22}=\frac{1}{2}\left[\begin{array}{cc}
    \frac{2 k_{\boldsymbol{R}}}{2-\psi_{2}} & c_{2} \\
    c_{2} & \lambda_{\max }({\boldsymbol{J}})
        \end{array}\right]$, $\psi_2 = \frac{1}{k_R}\mathcal{V}_2|_{t=0,c_2=0}$, and $W_{2}= \left[\begin{array}{cc}
\frac{c_{2} k_{\boldsymbol{R}}}{\lambda_{\max }({\boldsymbol{J}})} & -\frac{c_{2} k_{\boldsymbol{\omega}}}{2 \lambda_{\min }({\boldsymbol{J}})} \\
-\frac{c_{2} k_{\boldsymbol{\omega}}}{2 \lambda_{\min }({\boldsymbol{J}})} & k_{\boldsymbol{\omega}}-c_{2}
\end{array}\right]$.

The matrices $W_{2}$, $M_{21}$, and $M_{22}$ are positive-definite by choosing a positive constant $c_{2}$ such that
\begin{equation}
    c_{2} < \min \left\{k_{\boldsymbol{\omega}}, \frac{4 k_{\boldsymbol{\omega}} k_{\boldsymbol{R}}\lambda_{\min }(J)^{2}}{k_{\omega}^{2} \lambda_{\max }(\boldsymbol{J})+4 k_{\boldsymbol{R}} \lambda_{\min }(\boldsymbol{J})^{2}}, \sqrt{k_{\boldsymbol{R}} \lambda_{\min }(\boldsymbol{J})}\right\}
    \label{eq:c2}
\end{equation}

For translation dynamics, let a Lyapunov candidate $\mathcal{V}_{1}$ be
\begin{equation}
\mathcal{V}_{1}=\frac{1}{2} k_{\boldsymbol{p}} \left\|\boldsymbol{e}_{\boldsymbol{p}}\right\|^{2}+\frac{1}{2} m\left\|\boldsymbol{e}_{\boldsymbol{v}}\right\|^{2}+c_{1} \boldsymbol{e}_{\boldsymbol{p}}^\top \boldsymbol{e}_{\boldsymbol{v}}
\end{equation}
where $c_{1}$ is a positive constant. Clearly, we have
\begin{equation}
    z_{1}^{\top} M_{11} z_{1} \leq \mathcal{V}_{1} \leq z_{1}^{\top} M_{12} z_{1}
\end{equation}
where $z_{1}=\left[\begin{array}{c}\left\|\boldsymbol{e}_{\boldsymbol{p}}\right\| \\
\left\|\boldsymbol{e}_{\boldsymbol{v}}\right\|\end{array}\right]$, $M_{11}=\frac{1}{2}\left[\begin{array}{cc}
    k_{\boldsymbol{p}} & -c_{1} \\ 
    -c_{1} & m 
        \end{array}\right]$, and
$M_{12}=\frac{1}{2}\left[\begin{array}{cc}
    k_{\boldsymbol{p}} & c_{1} \\ 
    c_{1} & m
        \end{array}\right]$.
        
The derivative of $\mathcal{V}_{1}$ is given by
\begin{equation}
    \begin{aligned}
        \dot{\mathcal{V}}_{1}=& k_{\boldsymbol{p}} \boldsymbol{e}_{\boldsymbol{p}}^\top \boldsymbol{e}_{\boldsymbol{v}} + \boldsymbol{e}_{\boldsymbol{v}}^\top \left(-k_{\boldsymbol{p}} \boldsymbol{e}_{\boldsymbol{p}}-k_{\boldsymbol{v}} \boldsymbol{e}_{\boldsymbol{v}}\right) + c_{1} \boldsymbol{e}_{\boldsymbol{v}}^\top \boldsymbol{e}_{\boldsymbol{v}} +\frac{c_{1}}{m} \boldsymbol{e}_{\boldsymbol{p}}^\top \left(-k_{\boldsymbol{p}}  \boldsymbol{e}_{\boldsymbol{p}} - k_{\boldsymbol{v}} \boldsymbol{e}_{\boldsymbol{v}}\right) \\
        =&-\left(k_{\boldsymbol{v}} -c_{1}\right)\left\|\boldsymbol{e}_{\boldsymbol{v}}\right\|^{2}-\frac{c_{1} k_{\boldsymbol{p}} }{m}\left\|\boldsymbol{e}_{\boldsymbol{p}}\right\|^{2}-\frac{c_{1} k_{\boldsymbol{v}}}{m} \boldsymbol{e}_{\boldsymbol{p}}^\top \boldsymbol{e}_{\boldsymbol{v}} \\
        \leq& -z_{1}^{\top} W_{1} z_{1}
    \end{aligned}
\end{equation}
where $W_{1}=\left[\begin{array}{cc}
    \frac{c_1 k_{\boldsymbol{p}}}{m} & -\frac{c_1 k_{\boldsymbol{v}}}{2m} \\ 
    -\frac{c_1 k_{\boldsymbol{v}}}{2m} & k_{\boldsymbol{v}}-c_1
        \end{array}\right]$.
The matrices $W_{1}$, $M_{11}$, and $M_{12}$ are positive-definite by choosing a positive constant $c_{1}$ such that
\begin{equation}
    c_{1} < \min \left\{\frac{4 m k_{\boldsymbol{p}} k_{\boldsymbol{v}} }{4 m k_{\boldsymbol{p}} + k_{\boldsymbol{v}}^2}, \sqrt{m k_{\boldsymbol{p}}}\right\} .
    \label{eq:c1}
\end{equation}

Now let $\mathcal{V}=\mathcal{V}_{1}+\mathcal{V}_{2}$ be the Lyapunov candidate of the complete system,
\begin{equation}
    \mathcal{V} = \frac{1}{2} k_{\boldsymbol{p}} \left\|\boldsymbol{e}_{\boldsymbol{p}}\right\|^{2}+\frac{1}{2} m\left\|\boldsymbol{e}_{\boldsymbol{v}}\right\|^{2}+c_{1} \boldsymbol{e}_{\boldsymbol{p}}^\top \boldsymbol{e}_{\boldsymbol{v}} + \frac{1}{2} \boldsymbol{e}_{\boldsymbol{\omega}}^\top \boldsymbol{J} \boldsymbol{e}_{\boldsymbol{\omega}} + \boldsymbol{k}_{\boldsymbol{R}} \Psi\left(\boldsymbol{R}, \boldsymbol{R}_{d}\right) + c_{2} \boldsymbol{e}_{\boldsymbol{R}}^\top \boldsymbol{e}_{\boldsymbol{\omega}}
\end{equation}

The bound of the Lyapunov candidate $\mathcal{V}$ can be written as
\begin{equation}
    z_{1}^{T} M_{11} z_{1}+z_{2}^{T} M_{21} z_{2} \leq \mathcal{V} \leq z_{1}^{T} M_{12} z_{1}+z_{2}^{T} M_{22} z_{2}.  
\end{equation}

The time derivative of $\mathcal{V}$ is given by 
\begin{equation}
    \dot{\mathcal{V}} \leq -z_{1}^{T} W_{1} z_{1} - z_{2}^{T} W_{2} z_{2}.
\end{equation}

By choosing $c_1$ and $c_2$ such that \cref{eq:c1,eq:c2} are satisfied, we obtain that $\mathcal{V}$ is positive definite, and $\dot{\mathcal{V}}$ is negative definite. Therefore, we conclude that the zero equilibrium of the tracking errors of the complete dynamics is exponentially stable.

By \Cref{cor:controllability}, a rotor thrust configuration realizing the required wrench $\boldsymbol{w}_{\mathrm{Geo}} = [\boldsymbol{f}_{\mathrm{Geo}}^\top \; \boldsymbol{\tau}_{\mathrm{Geo}}^\top]^\top$ always exists, so the corrected thrust allocation is well-defined at every instant; together with \Cref{cor:wrench-preservation}, this establishes that \Cref{thm:exp-stable-corrected} holds for the actual implemented flight controller.
\end{proof}

We obtain exponential stability of the generalized geometric controller for almost every pair of $\left(\boldsymbol{R}(0), \boldsymbol{R}_{d}(0)\right).$ Specifically, \Cref{thm:exp-stable-corrected} holds for $\Psi\left(\boldsymbol{R}(0), \boldsymbol{R}_{d}(0)\right) < 2$, \ie the initial attitude error between $\boldsymbol{R}(0)$ and $\boldsymbol{R}_{d}(0)$ should be less than $180\degree$. In other words, we prove that the closed-loop controller is exponentially stable in \textit{almost} the entire $\mathsf{SO}(3)$. This result improves the exponential stability of the geometric controller of a regular quadrotor; that controller requires the initial attitude error between $\boldsymbol{R}(0)$ and $\boldsymbol{R}_{d}(0)$ to be less than $90\degree$~\cite{lee2010geometric}.

\subsection*{Proof of \Cref{prop:optimality}}
We prove \Cref{prop:optimality}: that the thrust distribution given by~\eqref{eq:inverse} is the optimal solution with respect to the cost function $E(\boldsymbol{t}) = \frac{1}{2}\sum_{i=1}^4 ||\boldsymbol{t}_i||_2^2$.

\begin{proof}[Proof of \Cref{prop:optimality}]
    The constraints for the problem are given by $\boldsymbol{f}=\sum_{i=1}^4 \boldsymbol{t}_i$ and $\boldsymbol{\tau}=\sum_{i=1}^4 \boldsymbol{l}_i\times\boldsymbol{t}_i$. The Lagrangian for the constrained optimization problem is given by
    \begin{align}
        \mathcal{L}\left(
            \boldsymbol{t}
        \right) &= \frac{1}{2}\sum_{i=1}^4 ||t_i||_2^2 + \lambda^{\intercal}\left(\boldsymbol{f}-\sum_{i=1}^4 \boldsymbol{t}_i\right) + \mu^{\intercal}\left(\boldsymbol{\tau}-\sum_{i=1}^4 \boldsymbol{l}_i\times\boldsymbol{t}_i\right).\label{eq:lagrangian}
    \end{align}
    The gradient of Lagrangian is $\nabla\mathcal{L}=\begin{bmatrix}\left(\frac{\partial\mathcal{L}}{\partial\boldsymbol{t}_1}\right)^\top & \left(\frac{\partial\mathcal{L}}{\partial\boldsymbol{t}_2}\right)^\top & \left(\frac{\partial\mathcal{L}}{\partial\boldsymbol{t}_3}\right)^\top & \left(\frac{\partial\mathcal{L}}{\partial\boldsymbol{t}_4}\right)^\top\end{bmatrix}^\top$.
    Without loss of generality, we expand $\frac{\partial\mathcal{L}}{\partial\boldsymbol{t}_1}$ using appropriate values for $\boldsymbol{r}_i$, and equating it with $0$, we get
    \begin{align*}
        \frac{\partial\mathcal{L}}{\partial\boldsymbol{t}_1} = \begin{bmatrix}
            \partial\mathcal{L}/\partial t_{1x}\\
            \partial\mathcal{L}/\partial t_{1y}\\
            \partial\mathcal{L}/\partial t_{1z}
        \end{bmatrix} = \begin{bmatrix}
            t_{1x} - \lambda_1 - \mu_3(-b)\\
            t_{1y} - \lambda_2 - \mu_3(l) \\
            t_{1z} - \lambda_3 - (\mu_1b-\mu_2l),
        \end{bmatrix} = 0 .
    \end{align*}
    Equivalently, we have
    \begin{equation*}
        \boldsymbol{t}_1 = \lambda + \begin{bmatrix}
            0 & 0 & -b \\ 0 & 0 & l \\ b & -l & 0
        \end{bmatrix}\mu.
    \end{equation*}
    
    Similarly, solving further for all $\frac{\partial\mathcal{L}}{\partial\boldsymbol{t}_2}=0$, $\frac{\partial\mathcal{L}}{\partial\boldsymbol{t}_3}=0$, $\frac{\partial\mathcal{L}}{\partial\boldsymbol{t}_4}=0$, and constraints $\boldsymbol{f}=\sum_{i=1}^4 \boldsymbol{t}_i$ and $\boldsymbol{\tau}=\sum_{i=1}^4 \boldsymbol{l}_i\times\boldsymbol{t}_i$, the solution for multipliers $\lambda$ and $\mu$ is
    \begin{align}
        \lambda = \frac{\boldsymbol{f}}{4}, \quad \quad \mu = \begin{bmatrix}
            \tau_x/4b^2 \\ \tau_y/4l^2 \\ \tau_z/4r^2
        \end{bmatrix}. \label{eq:mu_lambda_solution}
    \end{align}
    Substituting these in the solutions for $\boldsymbol{t}_i$, we get
    \begin{align*}
        \boldsymbol{t}_1 = \boldsymbol{f}/4 + \begin{bmatrix} 0 & 0 & -b/r^2 \\ 0 & 0 & l/r^2 \\ 1/b & -1/l & 0 \end{bmatrix} \boldsymbol{\tau}/4, &\quad \quad
        \boldsymbol{t}_2 = \boldsymbol{f}/4 + \begin{bmatrix} 0 & 0 & b/r^2 \\ 0 & 0 & l/r^2 \\ -1/b & -1/l & 0 \end{bmatrix} \boldsymbol{\tau}/4 \\
        \boldsymbol{t}_3 = \boldsymbol{f}/4 + \begin{bmatrix} 0 & 0 & b/r^2 \\ 0 & 0 & -l/r^2 \\ -1/b & 1/l & 0 \end{bmatrix} \boldsymbol{\tau}/4, &\quad \quad
        \boldsymbol{t}_4 = \boldsymbol{f}/4 + \begin{bmatrix} 0 & 0 & -b/r^2 \\ 0 & 0 & -l/r^2 \\ 1/b & 1/l & 0 \end{bmatrix} \boldsymbol{\tau}/4. 
    \end{align*}
    This is the same as the solutions proposed in the inverse mapping equations \eqref{eq:inverse}.

    Equivalently, we can use the linear map $\boldsymbol{M}:\boldsymbol{t}\rightarrow (\boldsymbol{f},\boldsymbol{\tau})$ to find the solution. Since we already proved that this is an underdetermined system with infinite solutions, we can use pseudo-inverse $\boldsymbol{M}^{\dagger}=\boldsymbol{M}^{\top}\left(\boldsymbol{MM}^{\top}\right)^{-1}$ to find the suitable $\boldsymbol{t}_i$ values, \ie
    \begin{equation}
        \boldsymbol{t} = \boldsymbol{M}^{\dagger} \begin{bmatrix}
        \boldsymbol{f}^\top & \boldsymbol{\tau}^\top
        \end{bmatrix}^\top. \label{eq:pseudoinverse} 
    \end{equation}
    The pseudo-inverse gives a minimum norm-squared solution among the infinite solutions, \textit{i.e.,} $\min\frac{1}{2}||\boldsymbol{t}||_2^2$, which is same as cost function, $E(\boldsymbol{t})$. The further expansion of the expression \eqref{eq:pseudoinverse} gives the same solution as the proposed inverse mapping.
\end{proof}


\end{document}